\documentclass[11pt]{article}

\usepackage[preprint,nonatbib]{neurips_2026}

\usepackage[numbers]{natbib}
\usepackage[utf8]{inputenc} 
\usepackage[T1]{fontenc}    
\usepackage{hyperref}       
\usepackage{url}            
\usepackage{booktabs}       
\usepackage{amsfonts}       
\usepackage{nicefrac}       
\usepackage{microtype}      
\usepackage[table]{xcolor}         

\usepackage{times}
\usepackage{latexsym}
\usepackage[T1]{fontenc}
\usepackage[utf8]{inputenc}
\usepackage{microtype}
\usepackage{inconsolata}
\usepackage{graphicx}
\usepackage{amsfonts}
\usepackage{amsmath}
\usepackage{amssymb}
\usepackage{amsthm}
\usepackage{mathtools}
\usepackage{booktabs}
\usepackage{array}
\usepackage{makecell}

\usepackage{multirow}

\newcommand{\R}{\mathbb{R}}
\newcommand{\E}{\mathbb{E}}
\newcommand{\Simp}{\mathbb{S}}

\DeclareMathOperator*{\argmax}{arg\,max}
\DeclareMathOperator*{\argmin}{arg\,min}
\newtheorem{definition}{Definition}
\newtheorem{proposition}{Proposition}
\newcommand{\Xspace}{\mathcal{X}}
\newcommand{\Dspace}{\mathcal{D}}
\newcommand{\Yspace}{\mathcal{Y}}

\newcommand{\ind}[1]{I\left(#1\right)}
\newcommand{\qvec}{\mathbf{q}}
\newcommand{\pvec}{\mathbf{p}}

\usepackage{xcolor}

\newcommand{\cola}[1]{\textcolor{red}{#1}}
\newcommand{\colb}[1]{\textcolor{green}{#1}}
\newcommand{\colc}[1]{\textcolor{blue}{#1}}
\newcommand{\cold}[1]{\textcolor{orange}{#1}}
\newcommand{\cole}[1]{\textcolor{brown}{#1}}

\def\ECUAS#1{$\text{ECUAS}_#1$}

\title{\ECUAS{n}: A family of metrics for principled evaluation of uncertainty-augmented systems}

\def\eqref#1{Equation (\ref{#1})}

\author{
  \textbf{Lautaro Estienne\textsuperscript{1,2,3,4,*}},
  \textbf{Erik Ernst\textsuperscript{2,*}},
  \textbf{Matías Vera\textsuperscript{1,5}},
  \textbf{Pablo Piantanida\textsuperscript{4,6,7}},
  \textbf{Luciana Ferrer\textsuperscript{2}}
\\
\\
  \textsuperscript{1}School of Engineering, UBA, Argentina \\
  \textsuperscript{2}ICC, CONICET-Universidad de Buenos Aires, Argentina \\
  \textsuperscript{3}LISN, CNRS, Université Paris-Saclay, France \\
  \textsuperscript{4}International Laboratory on Learning Systems, Canada \\
  \textsuperscript{5}CSC, CONICET, Argentina \\
  \textsuperscript{6}Mila - Quebec AI Institute, Canada \\
  \textsuperscript{7}CNRS, Université Paris-Saclay, France \\
  \textsuperscript{*}Equal contribution
\\
\\
  \small{
   \textbf{Correspondence:} \href{mailto:lestienne@fi.uba.ar}{lestienne@fi.uba.ar}, 
   \href{mailto:eernst@dc.uba.ar}{eernst@dc.uba.ar}, \href{mailto:lferrer@dc.uba.ar}{lferrer@dc.uba.ar}
  }
}

\begin{document}
\maketitle
\begin{abstract}
In high-stakes automated decision-making, access to predictive uncertainty is essential for enabling users -- human or downstream systems -- to accept or reject predictions based on application-specific cost trade-offs.
Such uncertainty-augmented (UA) systems -- i.e., systems that output both predictions and uncertainty scores -- are currently being assessed in the literature in a variety of ways, using separate metrics to evaluate the predictions and the uncertainty scores, setting a cost function with a fixed rejection cost or integrating over a coverage-risk curve. We argue that these evaluation approaches are inadequate for assessing overall performance of the UA system for decision making under uncertainty and propose a novel family of metrics, \ECUAS{n}, formulated as proper scoring rules for the task of interest. The parameter $n$ controls the trade-off between the cost of incorrect predictions and imperfect uncertainties depending on the needs of the use-case. We demonstrate the advantages of the \ECUAS{n} metrics both theoretically and empirically, through experiments on diverse classification and generation datasets, including a manually annotated subset of TriviaQA. 
\end{abstract}

\section{Introduction}

Modern machine learning systems are deployed across a wide range of settings, from automated decision-making pipelines to human-in-the-loop applications. In high-stakes scenarios, where incorrect outputs may have severe consequences, it is important for users to have access to information about the uncertainty with which a prediction is made, allowing them to accept or reject it depending on the error trade-off for their use case. As a result, a growing class of models now produce, alongside their candidate predictions, a measure of uncertainty or, a common counterpart, confidence. Properly evaluating these systems -- which we shall call uncertainty-augmented (UA) systems -- is therefore critical to ensure their usefulness, reliability, and safety.
We use the term UA system to refer both to traditional classifiers with a reject option -- sometimes called selective classifiers \cite{el-yaniv_foundations_2010} -- as well as modern generative systems that output an uncertainty or confidence score along with their generated answer \cite{lin2022teaching, tian-etal-2023-just, lin2024generating}. 
The goal of this paper is to develop a principled way to comprehensively assess the quality of UA systems grounded in decision theory which, as a large body of literature has argued \cite{good1952, savage1972foundations, Peterson_2017, russell2016artificial,dyrland2023doesevaluationstandevaluation}, is the rational way to make decisions and assess performance under uncertainty. 

In current literature, performance of UA systems is often measured through two or more separate metrics \cite{hendrycks_baseline_2017}. The quality of the candidate answers is assessed with accuracy or error rate, while the quality of the uncertainties is measured using binary classification metrics, taking the uncertainty as a predictor of incorrect answers. These metrics include the area under the ROC curve (AUC) \cite{duan2024, yang-etal-2025-maqa, xiong2024efficient,fadeeva-etal-2024-fact}, expected calibration error (ECE) \cite{jiang-etal-2021-know,stengel-eskin-van-durme-2023-calibrated,kapoor-etal-2024-calibration,gao-etal-2024-spuq}, Brier score \cite{kadavath2022, tian-etal-2023-just, liu2024litcab, ulmer-etal-2024-calibrating}, and cross-entropy \cite{NIPS2017_9ef2ed4b,malinin2021uncertainty,mielke-etal-2022-reducing,10.5555/3692070.3692835}. This evaluation strategy does not adequately reflect the value that the system provides to its user -- a human or a downstream stage -- who will use the uncertainties to decide whether to accept or reject the candidate answer. What is needed is a metric that reflects the quality of that final decision given by the candidate answer or a rejection. This can naturally be done based on the principles of decision theory.

Perhaps the earliest work taking a decision theoretical approach to classification with a reject option was Chow's 1957 paper on character recognition~\cite{chow_optimum_1957}. In that work, performance is measured through a cost function which includes costs for accepted incorrect answers depending on the ground truth, as well as a cost for rejection. Decisions are then made by minimizing the expectation of this cost over the predictive class distribution provided by a classifier. 
Later, El-Yaniv et al.~\cite{el-yaniv_foundations_2010} showed that the this cost function can be expressed as a combination of two values: the selective risk, defined as the expected misclassification cost of accepted predictions and the coverage, given by the probability that the prediction is accepted. Several other works followed Chow's proposal, using or studying the cost-based approach for evaluation in selective classification, e.g.,  \cite{geifman_selective_2017,franc_optimal_2023,Charoenphakdee2020,Bartlett2008}.

The cost-based approach to evaluation of UA system, while simple and principled, is often seen as problematic since it requires the specification of a rejection cost which may vary with the application or across different users or even instances. To solve this problem, Nadeem et al.~\cite{nadeem10} proposed to compute the area under the accuracy vs coverage curve, obtained by sweeping a threshold over the uncertainty and integrating over the resulting curve. This was later generalized to the area under the risk-coverage curve (AURC)~\cite{geifman2018biasreduced}, a widely used metric in the literature of UA systems \cite{geifman2018bias, jaeger2022call, bungert2023understanding, cheng2023unified, zhu2023openmix, varshney2020towards, naushad2024super, van2024beyond, van2023document, zhu2022rethinking, xin2021art, yoshikawa2023selective, ding2020revisiting, zhu2023revisiting, galil2021disrupting, franc2023optimal, cen2023devil, xia2022augmenting, cattelan2023fix, tran2022plex, kim2023unified, ashukha2020pitfalls, xia2024understanding}. Subsequent works proposed variants of the AURC to address some of its  shortcoming~\cite{zhou_novel_2025,traub_overcoming_2024}. Crucially, the AURC and its variants do not assume that the uncertainties will be given a probabilistic interpretation. In particular, it does not penalize uncertainties that are under or overestimated, as long  as the ranking is preserved. We see this as a major flaw, since uncertainties that do not faithfully reflect the probability of correctness are not easy to interpret. In this work, we construct evaluation metrics for UA systems that, like AURC, do not commit to a specific rejection cost but that, unlike the AURC, do reward uncertainties that are interpretable probabilistically. 

Our contributions can be summarized as follows:

\textbf{A principled replacement for the AURC grounded in decision theory.} In this work, we define the expected cost for uncertainty augmented systems (\ECUAS{n}), a family of metrics to evaluate the quality of UA systems based on first principles from Bayes decision theory~\cite{savage1961foundations}. The \ECUAS{n} metrics are given by the expectation of a cost function which is constructed as an integral over a family of cost functions parameterized by the rejection cost.
Crucially, these cost functions are proper scoring rules (PSR)~\cite{hendrickson1971proper, savage1972foundations,gneiting_strictly_2007}, a family of metrics proposed decades ago for the assessment of predictive distributions. As such, they reward systems that produce probabilistic uncertainties, a property that is missing from the widely-used AURC metric. The \ECUAS{n} metrics can be used for evaluation of  UA classifiers and UA generative systems.

\textbf{Insight on optimality of uncertainty scores.} 
We provide a theoretical explanation for the empirical observation reported in prior works that better uncertainty scores can be obtained for generative systems using the distribution over semantic equivalence classes rather than over individual predictions \cite{Farquhar2024,kuhn2023semantic}. To our knowledge, this is a novel theoretical explanation for this empirical finding.

\textbf{Empirical validation} We compare our proposed metrics with traditional metrics for this task on classification and generation tasks. For the latter, we use TriviaQA with manually curated correctness labels. We analyze several desirable properties of the \ECUAS{n} metrics and discuss the selection of the value of the $n$ parameter depending on the use-case. The code for our experiments can be found in \url{https://github.com/erikernst4/callm} and the manual TriviaQA annotations are available as supplementary material.

\section{The ECUAS family: Decision-theoretical metrics for UA systems}\label{sec:metric_construction}

Formally, we define the UA system under evaluation as a function $f^{\text{UA}}:\Xspace \rightarrow \tilde \Dspace \times \R_{\geq 0}$ that receives an input $x\in\Xspace$  and outputs a pair $(\tilde d, u)$ composed by a candidate response $\tilde d \in \tilde \Dspace$ as well as an uncertainty score $u \in \R_{\geq 0}$ related to that decision. Our goal is to construct a metric -- or, rather, as we will see, a family of metrics -- to assess the quality of these combined outputs. To that end, we will assume that the system's uncertainty $u$ will be used -- by an end user or by a downstream stage -- to accept or reject the output $\tilde d$, so that the final decision can take one of two values: the candidate decision $\tilde d$ or the reject event $d_r$. That is, the final decision is given by a function $h\triangleq d^{\text{UA}} \circ f^{\text{UA}}$, where $d^{\text{UA}}: \tilde \Dspace \times \R_{\geq 0} \rightarrow \Dspace$, where $\Dspace \triangleq \tilde \Dspace \cup \{d_r\}$. 

We start, like Chow \cite{chow_optimum_1957}, by defining a function $C_\gamma:\Yspace \times \Dspace \rightarrow \R_{\geq 0}$ which quantifies the cost for the final decision $d\in\Dspace$ given the sample's class $y\in\Yspace \triangleq \{y_1, \ldots, y_K\}$:
\begin{align}\label{eq:reject_cost}
    C_{\gamma}(y,d) = 
    \begin{cases}
        \tilde C(y,d) & \text{if } d \in \tilde \Dspace \\
        \gamma & \text{if } d = d_r
    \end{cases}
\end{align}
where $\gamma\in\R_{\geq 0}$ is a cost assigned to the decision to reject the system's candidate prediction and $\tilde C: \Yspace \times \tilde \Dspace \rightarrow \R_{\geq 0}$ is the cost of the candidate response $\tilde d\in\tilde\Dspace$ given the ground truth $y\in\Yspace$.

The discussion on how $\tilde C$ should be selected is out of the scope of this work, though it has been widely discussed in many prior publications \cite{Raiffa_1970,russel_2010,Kahneman:2011,Hunik_2014,Dyrland_2022}. As for $\gamma$, as discussed above, while a specific value can be chosen in some scenarios, in many cases we would like to leave this decision to the final user rather than fixing it during development. In section \ref{sec:psr_integral}, we construct a general metric that does not rely on a fixed $\gamma$.

Once the cost function for a given sample $x$ is defined, an evaluation metric for assessing the decision system $h$ defined above can be computed as the expectation of this cost with respect to the empirical distribution over an evaluation set $\{(x^{(i)},y^{(i)})\}_{i=1}^N$:
\begin{align}\label{eq:ecuas_gamma}
    \mathbb{E}[C_\gamma(y, h(x))] & = \frac{1}{N}\sum_{i=1}^{N} C_\gamma(y^{(i)}, h(x^{(i)}))  = R_{\gamma}\phi_{\gamma} + \gamma (1-\phi_{\gamma}) 
\end{align}
where $\phi_{\gamma} = 1/N \sum_i I(h(x^{(i)}) \neq d_r)$ is the coverage, the fraction of samples for which the candidate answer was accepted, and $R_{\gamma} = 1/(N\phi_\gamma)\sum_i  I(h(x^{(i)}) \neq d_r)\ C_\gamma(y^{(i)}, h(x^{(i)}))$  is the selective risk, the average cost over accepted samples \cite{el-yaniv_foundations_2010}. 
This expected cost measures the performance of a system given by the composition of a UA system $f^\text{UA}$ and a decision function $d^\text{UA}$, for a fixed $\gamma$. The  function $d^\text{UA}$ is a simple thresholding operation, rejecting the candidate answer when $u$ is larger than a threshold $t$ and accepting it otherwise. The threshold can be tuned empirically to optimize the cost above or set according to decision theory as described below.

The AURC metric was proposed in \cite{geifman2018biasreduced} as a way to comprehensively assess the performance of a UA system, without fixing $\gamma$ or $t$. It is computed as the area under the risk-coverage curve obtained by sweeping the threshold $t$. Notably, by design, this metric is invariant to monotonic transformations of the uncertainty score and, as a consequence, it does not reward UA systems that provide uncertainties that can be interpreted probabilistically. Here, we propose a novel family of metrics that, as the AURC, are comprehensive assessors of the UA system performance but that, unlike AURC, do reward interpretable uncertainty scores.
The metric is obtained by first constructing a PSR by minimizing the expected cost above (Section \ref{sec:psr_for_cgamma}), and then integrating the resulting PSRs over $\gamma$ (Section \ref{sec:psr_integral}).

\subsection{$C^*_\gamma$: A proper scoring rule derived from $C_\gamma$}
\label{sec:psr_for_cgamma}

Given an input $x\in\Xspace$ and a cost function $C$, the optimal decision under uncertainty is  that which minimizes the expectation of the cost function with respect to the conditional distribution over the classes for that input $x$, $\qvec : \Xspace \rightarrow \Simp_K$, where $\Simp_K$ refers to the simplex of dimension $K-1$:
\begin{equation}
  d_B(\qvec)\in \argmin_{d\in\Dspace} \E_{y\sim \qvec}[C_\gamma(y,d)] 
\end{equation}
A decision made in this way is called a Bayes decision.
Note that $\qvec$ always depends on the input $x$, though sometimes, as above, we suppress this dependence for notational clarity.

For the cost defined in \eqref{eq:reject_cost}, the Bayes decisions are given by
\begin{align}\label{eq:dB_final}
    d_B(\qvec) =
    \begin{cases}
        \tilde d_B(\qvec) & \text{if } u_{\tilde C}(\qvec) \leq \gamma \\
        d_r & \text{if } u_{\tilde C}(\qvec) > \gamma 
    \end{cases}
    \ \text{where}
\left.\begin{array}{@{}ll@{}}
  & u_{\tilde C}(\qvec) = \sum_{k=1}^K q_k \tilde C(y_k,\tilde d_B(\qvec)) \\
  & \tilde d_B(\qvec) \in \argmin_{\tilde d\in\tilde \Dspace} \E_{y\sim \qvec}[\tilde C(y,\tilde d)] 
\end{array}\right.
\end{align}

In the literature of selective classification, the expected cost $u_{\tilde C}(\qvec)$ is called {\bf uncertainty}.
The proof of this result is well-known \cite{franc_optimal_2023}, but it can also be found in Appendix  \ref{app:bayes_decisions_for_c_gamma} for completeness. 

\eqref{eq:dB_final} shows that Bayes decisions for our cost of interest  $C_\gamma$ can be made in two steps: for a given $\qvec$, first choose the Bayes decision $\tilde d_B(\qvec)$ corresponding to $\tilde C$, and then decide whether to accept or reject it based on the uncertainty value $u_{\tilde C}(\qvec)$ by comparing it to the threshold $\gamma$. 
That is, access to the full $\qvec$ is not required to make optimal decisions for $C_\gamma$.
Hence, a system which, for a given input $x$, outputs $f^{\text{UA}}(x)=(\tilde d, u)$,  with $\tilde d = \tilde d_B(\qvec(x))$ and $u=u_{\tilde C}(\qvec(x))$, can be used to make Bayes decisions according to $d_B$ above.  
If, on the other hand, $\tilde d$ or $u$ are not computed in this way, the decisions made with \eqref{eq:dB_final} by replacing $\tilde d_B(\qvec)$ with the system's $\tilde d$ and $u_{\tilde C}(\qvec)$ with the system's $u$, may be suboptimal for the given $\tilde C$ and $\gamma$.

This leads us to an interesting insight: When designing a form for the uncertainty, it is essential to consider the cost that will eventually be optimized by the user. When $\tilde C$ is given by  the 0-1 cost, $\tilde C_{01}(y_k, \tilde d_j) \triangleq I(y_k \neq \tilde d_j)$ with $\tilde \Dspace \equiv \Yspace$ where $\ind{\cdot}$ is the indicator function, the uncertainty is given by $u_{\tilde C}(\qvec) = 1 -q_e$, where $q_e\triangleq \max_k q_k$ is commonly called {\bf confidence} (see Appendix \ref{sec:proof_w_n01}). This is the value that is produced by many UA systems \cite{surveyuq_he}. If, on the other hand, we take the case in which $\tilde \Dspace = \Simp_K$ and $\tilde C(y_k,\qvec) = -\log(q_k)$ is the logarithmic loss, then $u_{\tilde C}(\qvec)$ is given by the {\bf entropy} of $\qvec$, which is also a commonly used uncertainty score \cite{surveyuq_he,macedo_entropy}. Importantly, systems that output a confidence or the entropy can be used to optimize the cost for the corresponding $\tilde C$, the 0-1 cost or the logarithmic loss, respectively. {\bf If the uncertainty calculation is mismatched to the $\tilde C$ of interest, the decisions made with those uncertainties may be suboptimal.} 

Given the expression for the Bayes decisions, for any cost function $C$, we can define a new cost function $C^*: \Yspace \times \Simp_K \rightarrow \R_{\geq 0}$, given by the cost of the Bayes decision made with $\qvec$: $C^*(y,\qvec)\triangleq C(y,d_B(\qvec))$. $C^*$ satisfies $\E_{y\sim \pvec}[C^*(y,\pvec)] \leq \E_{y\sim \pvec}[C^*(y,\qvec)]$ for every $\qvec,\pvec\in\Simp_K$, the defining property PSRs. 
Using a PSR as evaluation metric ensures that better probabilities are rewarded over poorer ones, a characteristic that we believe to be essential for UA systems that will be used for decision-making under uncertainty either by the final user or by a downstream stage.

The PSR corresponding to $C_\gamma$ is obtained by plugging \eqref{eq:dB_final} in \eqref{eq:reject_cost}:
\begin{align}\label{eq:psr_gamma_fixed}
C_\gamma^*(y,\qvec)\triangleq C_\gamma(y, d_B(\qvec)) = 
    \begin{cases}
        \tilde C(y,\tilde d_B(\qvec)) & \text{if } u_{\tilde C}(\qvec) \leq \gamma \\
        \gamma & \text{if } u_{\tilde C}(\qvec) > \gamma 
    \end{cases}
\end{align}
The expectation of this cost w.r.t. the empirical distribution can be used to assess the quality of UA systems for a given $\gamma$ by setting $\tilde d_B(\qvec) = \tilde d$ and $u_{\tilde C} = u$  in the right hand side, where $\tilde d$ and $u$ are the outputs of the UA system under evaluation, as defined above. 
This value will be optimal for a given $\qvec$ if the outputs of the UA system are computed -- internally, by the system -- according to~\eqref{eq:dB_final}. Otherwise, as we illustrate in Appendix \ref{app:classification_expts}, the expected cost increases with respect to that optimal value, adequately reflecting the sub-optimality of the system's outputs.

\subsection{$C_w$: A cost for scenarios in which the reject cost is unknown}
\label{sec:psr_integral}

\eqref{eq:psr_gamma_fixed} depends on a fixed cost $\gamma$ for rejecting a sample. This cost should be determined based on the needs of the application of interest, which may not be known at development time. Since the cost depends on the uncertainty only through thresholding with $\gamma$, two uncertainty functions that are monotonically related and coincide at the threshold would result in the same cost. Hence, a system that has a low cost for a given $\gamma=\gamma_0$ could have a high cost for a different $\gamma=\gamma_1$. We want a cost function that encourages the uncertainty to be useful for decision making for \emph{all} values of $\gamma$. In this section, we show how to construct a PSR that satisfies this condition.

An important property of PSRs is that a weighted combination of PSRs is also a PSR. That is, it can be shown that if $C_\gamma^*$ is a PSR and $w:[0,u_M] \rightarrow \R_{\geq 0}$ is an integrable non-negative function over the interval $[0,u_M]$ with $u_M \triangleq \max_{\qvec \in \Simp_K} u_{\tilde C}(\qvec)$, then
\begin{align} \label{eq:psr_integral}
    C^*_w(y,\qvec) & \triangleq \int_{0}^{u_M} w(\gamma)C^*_\gamma(y,\qvec)d\gamma \  = \int_0^{u_{\tilde C}(\qvec)}\gamma w(\gamma)d\gamma + \tilde C(y,\tilde d_B(\qvec))\int_{u_{\tilde C}(\qvec)}^{u_M}w(\gamma)d\gamma
\end{align}
is also a PSR (see, e.g., \cite{brummer_thesis}, and Appendix \ref{sec:proof_psr}). Values of $\gamma$ larger than $u_M$ would result in acceptance of all samples, regardless of $\qvec$. Hence, the integral is restricted to values of $\gamma$ that may result in at least some samples being rejected for some systems. 

We can now construct a variety of PSRs for our problem of interest by selecting different $w$ functions, which determine the importance of each $\gamma$ value. For instance, a Dirac delta function $\delta(\gamma-\gamma_0)$ at a specific $\gamma_0$ would recover $C_{\gamma_0}^*$. On the other hand, a $w$ that is non-zero between 0 and $u_M$ results in a weighted combination of $C_\gamma^*$ across all valid $\gamma$ values. A low value for such metric would indicate that the system works well for all values of~$\gamma$. 

For the experiments in this paper, we will focus on a family of metrics that is obtained when $w$ is parameterized as $w_n(\gamma) \triangleq \alpha_n \gamma^{n-1}$, for $n \geq 0$, where $\alpha_n \triangleq (n+1) u_M^{-(n+1)}$ is a scaling factor to make the cost equal 1 when the uncertainty is $u_M$. Plugging $w_n$ in \eqref{eq:psr_integral}, we get 
\begin{align}\label{eq:psr_integral_n}
    C^*_n(y,\qvec)\triangleq C^*_{w_n}(y,\qvec) = 
    \begin{dcases}
    \alpha_n \ u_{\tilde C}(\qvec) + \alpha_n  \left(\log u_M - \log u_{\tilde C}(\qvec)\right)\tilde C(y,\tilde d_B(\qvec)) & \text{if }  n=0\\
    \frac{\alpha_n}{n+1} u_{\tilde C}(\qvec)^{n+1} + \frac{\alpha_n}{n} \left(u_M^n - u_{\tilde C}(\qvec)^n \right)\tilde C(y,\tilde d_B(\qvec)) & \text{if }  n>0\\
    \end{dcases}
\end{align}
As for $C^*_\gamma$, the $C^*_n$ cost also depends on $\qvec$ only through the candidate answer, $\tilde d_B(\qvec)$, and the uncertainty, $u_{\tilde C}(\qvec)$. The derivation of this cost can be found in Appendix~\ref{sec:proof_w_n}.

For the special case in which $\tilde C$ is the 0-1 cost defined above, the Bayes decisions are given by $\tilde d_B(\qvec) = e \triangleq \arg \max_k q_k$, the uncertainty is given by $u_{\tilde C}(\qvec)  = 1 - q_e$ with $q_e=\max_k q_k$, and $u_M = 1 - q_m$ where $q_m=1/K$ is the minimum value that the posterior can take for a Bayes decision for a $K$-class problem. The expression for $C^*_n$ for this case is given by (see Appendix \ref{sec:proof_w_n01})
\begin{align} \label{eq:psr_integral_n01}
    C^*_{n01}(y,\qvec) = 
    \begin{dcases}
    \alpha_n (1-q_e) +  \alpha_n \left(\log (1-q_m) - \log (1-q_e)\right) I(y\neq e) & \text{if }  n=0 \\
    \frac{\alpha_n}{n+1} (1-q_e)^{n+1} + \frac{\alpha_n}{n} \left((1-q_m)^n - (1-q_e)^n\right)I(y\neq e) & \text{if }  n>0 
    \end{dcases}
\end{align}

Figure \ref{fig:cncag_vs_qe} shows the value of this cost as a function of $q_e$, for correct and incorrect candidate answers. The minimum value of $q_e$ in these plots is $q_m$: smaller values would be inconsistent with the fact that $e$ is the argmax decision. We show the plots for various $K$, including $K\rightarrow \infty$, in which case $q_m=0$. For all values of $n$, the minimum cost for an incorrect answer and the maximum cost for a correct answer are both 1, which occurs when the sample is rejected for having a $q_e=q_m$. Also for all $n$, the cost monotonically increases with the confidence for incorrect answers and decreases for correct answers -- it rewards systems that produce high confidence for correct answers and low confidence for incorrect ones.

The $n$ parameter adjusts the impact of the quality of the confidence in the cost value. When $n=0$, the metric penalizes very heavily large values of $q_e$ given to incorrect candidate answers. On the other hand, larger values of $n$ give milder penalties to such cases. Similarly, for correct cases, the impact of a low $q_e$ decreases as $n$ increases. 
In one extreme, when $n\rightarrow \infty$, the metric degenerates to the 0-1 cost, penalizing equally all answers of the same type, regardless of $q_e$, except when $q_e=q_m$, in which case both correct and incorrect candidates receive a cost of 1, since they would be rejected for any value of the rejection cost $\gamma$. In the other extreme, when $n=0$, the metric is much more sensitive on the confidence being useful for making accept/reject decisions. This is the setting we recommend for cases in which accepting an incorrect answer would have severe consequences.

\begin{figure}
    \centering
    \includegraphics[width=1.01\linewidth]{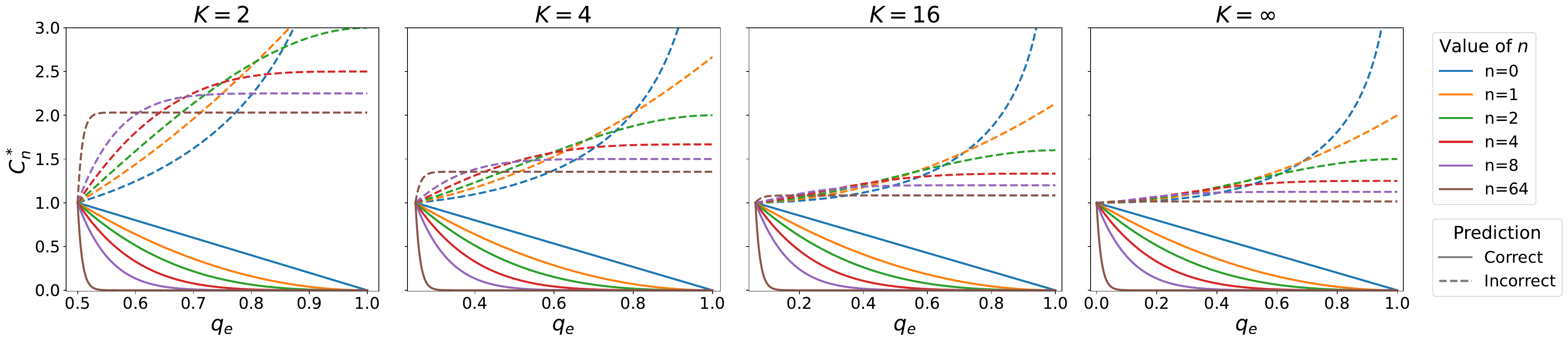}
    \caption{$C_{n}^*$ as a function of the confidence $q_e$, when candidate decisions are correct (solid lines) and incorrect (dashed lines), for different values of $n$, the parameter in $w$, and $K$, the number of classes.}
    \label{fig:cncag_vs_qe}
\end{figure}
 
Given an evaluation set $\{(x^{(i)},y^{(i)})\}_{i=1}^N$, we compute the expected cost for the uncertainty augmented system (ECUAS) as the mean of $C_n^*$ over this set:
\begin{align}
\text{ECUAS}_n = \frac{1}{N}\sum_{i=1}^{N} C^*_n(y^{(i)}, \qvec(x^{(i)})). \label{eq:ecuas}
\end{align}

\section{Application of ECUAS to generative systems}
\label{sec:generative}
\vspace{-2mm}

An important family of UA systems is that based on generative models \cite{survey-llm-nacl-2024,survey-llm-acl-2025,survey-llm-kdd-2025}. To use the ECUAS metrics in this scenario, we need to adapt the definition of $\tilde C$. 
While in classification scenarios it is usually assumed that there is a single correct decision for each class, in the generation scenario many different predictions may be correct for a given prompt. Further, $\Yspace$, the set of classes, cannot be explicitly defined for most applications for which generative systems are used.
To define $\tilde C$, we take $\Yspace$ to be the set of equivalence classes where each $y_k$ corresponds to a set of possible responses that are equivalent to each other. Finally, $\tilde \Dspace$ is the extremely large space of possible responses by the generative system. 
In practice, as we explain below, neither $\Yspace$ nor $\tilde \Dspace$ need to be specified explicitly.
 
Given these $\Yspace$ and $\tilde \Dspace$, we define the cost function for the candidate answers as $\tilde C(y,\tilde d) = I(y \neq \hat y(\tilde d))$, where $\hat y(\tilde d)$ is the equivalence class for $\tilde d$. This is a simple generalization of the 0-1 cost described above to the particular problem of generative systems. The average of this cost is the standard way in which error rate is computed in the literature of UA generative systems \cite{hendrycks_baseline_2017,surveyuq_he,traub_overcoming_2024}. Note that all we need to know to compute this cost is whether the candidate answer is considered correct or incorrect given the prompt and context. 
For this $\tilde C$, the uncertainty is given by $u_{\tilde C}(\qvec) = 1-q_e$, with $q_e = \max_k q_k$, as for the standard 0-1 cost. The Bayes decisions are given by $d_B(\qvec) \in \argmax_k q_k$, i.e., any element from most likely equivalence set (see Appendix \ref{sec:proof_w_n01g}). 

This leads us to a key insight: in order to make optimal accept/reject decisions for the generalized error rate, the confidence $q_e$ needs to be computed as the probability of the equivalence class corresponding to the candidate answer, $\hat y(d)$, rather than as the confidence of the specific response $d$. Similarly, if we were using the logarithmic cost instead of the 0-1 cost, the uncertainty for optimal decision-making would be given by the entropy of $\qvec$ computed over equivalence classes.
This implies that systems that compute the uncertainty or the confidence from the distribution over possible responses are potentially suboptimal.
This observation has been made empirically in prior work \cite{kuhn2023semantic, Farquhar2024}. Our work provides a novel theoretical explanation for these prior empirical results.

As observed above, we do not need to specify $\Yspace$ or $\tilde \Dspace$ explicitly. We only need to be able to determine correctness of a candidate answer, the same label that is needed to compute the standard metrics reported in the literature. Yet, we do need to know $K$, the number of possible equivalence classes, which is required to compute $q_m$ and $\alpha_n$.  This number is unknown in most scenarios in which generative systems are used, but it can usually be assumed to be large. 
In our experiments, we will take $K\rightarrow \infty$ to compute the metrics for generative systems, though much smaller values of $K$ would lead to similar results as shown in Figure \ref{fig:cncag_vs_qe}.

\section{Empirical analysis}
\label{sec:experiments}
\vspace{-2mm}

In this section, we show results for three instances of the proposed \ECUAS{n} metric, \eqref{eq:ecuas}, setting $n$ equal to 0, 1 and 128 and $\tilde C$ to the standard 0-1 cost for the classification experiments and the generalized 0-1 cost for the generation experiments. We compare these results with the main metrics used in the literature of UA systems which, as described in the introduction, mostly assess the quality of the systems with separate metrics for the quality of the answers and the uncertainties. To evaluate the quality of the candidate answer the most common metric in this literature -- by a large margin -- is accuracy (ACC), or error rate ({\bf ER} = 1-ACC). For the uncertainty or the confidence, the most common metrics are: area under the ROC curve ({\bf AUC}), which measures only the ability of the score to separate correct from incorrect answers; expected calibration error ({\bf ECE}), which measures the degree to which the confidences could be made into better probabilities by transforming them with a histogram binning transform; and the Brier score ({\bf $\text{BS}_{q_e}$}) and Cross entropy ({\bf $\text{CE}_{q_e}$}), two PSRs which measure both discrimination and calibration quality of the confidences \cite{ferrer2025evaluating}. Note that, as the \ECUAS{n}, the $\text{BS}_{q_e}$ and the $\text{CE}_{q_e}$ are also PSRs but they are computed only over the confidence, not over the combined system output including the candidate answer. Hence, they do not adequately reflect the performance of the system as will be experienced by the user. Finally, we include the {\bf AURC}, the main metric  used in the literature to assess the quality of the full system.

It is important to note that, despite it being a wide-spread practice in the literature \cite{stengel-eskin-van-durme-2023-calibrated,NEURIPS2018_285e19f2,kapoor-etal-2024-calibration,kumar2019calibration,liu2024litcab}, confidence-based metrics like the ECE, AUC, $\text{BS}_{q_e}$ and $\text{CE}_{q_e}$, should not be compared across systems unless their candidate predictions are identical. This is because the ground truth labels used to compute them are given by whether the candidate predictions are correct or not. Hence, {\bf comparing the AUC of the confidence between two systems with different candidate predictions is the same as comparing the AUC of two systems on two different datasets}. One AUC may be better than the other because the data was easier, not because the system is better. This, along with the fact that these metrics do not reflect the overall performance of the system, makes these metrics unfit for evaluation of UA systems. We nevertheless show these metrics in our tables to illustrate how conclusions based on them may greatly differ from the conclusions obtained based on the principled metrics proposed in this work, highlighting the importance of choosing the right metric for the task of interest. 

For systems that explicitly compute the full $\qvec$, as in classification, it is also possible to compute BS and CE for the full $\qvec$ rather than just for $q_e$. We call these metrics $\text{BS}_{\qvec}$ and $\text{CE}_{\qvec}$. These PSRs reflect the quality of the underlying distribution which is then used to choose the candidate answer and produce the uncertainty. Yet, as discussed above, the full $\qvec$ is not available for generative systems so those metrics cannot be computed in that scenario. Further, while they are also PSRs, they do not directly reflect the performance for use-cases with a reject option. Hence, even when the full $\qvec$ is available, for such scenarios, the \ECUAS{n} metric may be more appropriate.

Below we show results for both classifier and generative model-based UA systems. 
Additional results can be found in Appendix \ref{app:classification_expts} and \ref{app:generative}. In particular, in Appendix \ref{app:classification_expts}, we explore the impact of the quality of the candidate answers on the ECUAS metrics, showing that they have the expected behavior, monotonically increasing as the quality of the decisions degrade. 

\vspace{-2mm}

\subsection{Analysis on UA classification systems}
\label{sec:experiments_classification}
\vspace{-2mm}
We show results on a variety of systems taken from \cite{ferrer2025no} for language, image and speech processing tasks. For each system, we show results on the raw scores provided by the model and on scores calibrated using a post-hoc affine calibration approach. Calibration improves the quality of the predictive distribution $\qvec$, improving the quality of both candidate predictions and confidences.
The candidate answers are selected using the argmax rule. Table \ref{tab:classification_subset} shows a subset of the results. A larger table of results along with a description of the datasets and the calibration procedure can be found in Appendix~\ref{app:classification_expts}. For this section, we report normalized versions of the ECUAS, ER, both $\text{BS}$, and both $\text{CE}$, obtained by dividing the raw metric value by the value for a naive system that always outputs the empirical prior distribution on the test data $\qvec_0$, as suggested in \cite{ferrer2025evaluating}. This normalization conveniently scales the metric values so that a value of 1.0 indicates that the system is as poor as the naive system. 

Table \ref{tab:classification_subset} shows that, as explained in Section \ref{sec:psr_integral}, the proposed cost when $n$ is large tends to the 0-1 cost so that its expectation is similar to the error rate (in this case, both normalized). For smaller values of $n$, the metric becomes sensitive to the quality of the confidences, integrating into a single value the impact of both incorrect candidate predictions and imperfect confidences.
When $n=0$ the ECUAS behavior becomes similar to that of the CE$_{q_e}$, with differences due to the fact that \ECUAS{0} is also affected by the quality of the candidate answers rather than just the confidences.  Analogously, the \ECUAS{1} shows similar but not identical trends to the BS$_{q_e}$ metric. Compare, for example, the trends for the raw CIFAR-10 scores where ResNet20 is better than VGG19 in terms of CE$_{q_e}$ and \ECUAS{0} (in \cola{red}), but worse in terms of BS$_{q_e}$ and \ECUAS{1} (in \colb{green}). These latter two metrics penalize the miscalibration in the raw VGG19 scores, which are overconfident due to overfitting of the model being 100 times larger than the ResNet20 model, less than the former two.

Comparing \ECUAS{0} and \ECUAS{1}  with CE$_{\qvec}$ and BS$_{\qvec}$, respectively, we again see similar trends. This is expected since they are all PSRs and good-quality $\qvec$, as evaluated by CE$_{\qvec}$ and BS$_{\qvec}$, results in good-quality candidate predictions and confidences, as evaluated by \ECUAS{0} and \ECUAS{1}. Yet, while the trends are similar, the values differ. For example, while CE$_{\qvec}$ for the raw scores for AGNews is 0.814, safely away from the naive system's performance, the \ECUAS{0} for that system is 1.005, indicating that its performance is worse than that of the naive system (results marked in \colc{blue}). That is, for our specific use case of classification with a reject option with the 0-1 cost for accepted answers, this system has the same performance as a naive system and, hence, is useless in practice. This fact is not diagnosed by the PSRs computed over $\qvec$.

Finally, we can see that the AURC is mostly unaffected by calibration even when the ECE greatly improves (compare \emph{raw} to \emph{cal} lines for every system, for example, cells in \cold{orange}), except when the transform also affects the quality of the decisions, which happens when it includes a non-zero shift term (see \cole{brown} cells).  
This is because AURC only rewards improvements in the quality of the confidences through the impact that it has on the decisions. When the calibration transform does not improve the quality of the decisions, the AURC is unchanged. This is the case when the transform's shift term is close to zero, degenerating into temperature scaling, keeping the ranking intact and, in particular, keeping the argmax decision intact. This invariance makes this metric inadequate for evaluating confidences that will be interpreted as correctness probabilities.

\begin{table}[t]
\centering
\caption{Standard and proposed metrics on classification systems. The prefix ``N-'' indicates normalization. The first line in the table indicates the variable that is evaluated by the metrics in each block. ER evaluates the candidate predictions ($\tilde d$), AUC, ECE, BS$_{q_e}$, and CE$_{q_e}$ evaluate different aspects of the confidence ($q_e$), BS$_\qvec$ and CE$_\qvec$ evaluate the full predictive distribution ($\qvec$), and AURC and \ECUAS{n} jointly evaluate the candidate answers and the confidences.}
\resizebox{\columnwidth}{!}{%
\begin{tabular}{lll|c|cccc|cc|c|ccc}
\toprule
& & & $\tilde d$  & \multicolumn{4}{c|}{$q_e$}  & \multicolumn{2}{c|}{$\qvec$} & $\tilde d$, $q_e$ & \multicolumn{3}{c}{$\tilde d$, $q_e$} \vspace{1mm} \\
 &  &  &  &  &  &  &  &  &  &  & \multicolumn{3}{c}{N-\ECUAS{n}} \\
  {\bf Data}  & \multicolumn{2}{l|}{\bf Model \ \ \ \ \ \ \ \ \ Score} & \textbf{N-ER} & \textbf{ECE} & \textbf{AUC} & \textbf{N-CE$_{q_e}$} & \textbf{N-BS$_{q_e}$} & \textbf{N-CE$_\qvec$} & \textbf{N-BS$_\qvec$} & \textbf{AURC} & \textbf{n=0} & \textbf{n=1} & \textbf{n=128} \\
\midrule
 \multirow[t]{2}{*}{FVCAUS} & \multirow[t]{2}{*}{PLDA} & raw & 3.9155 & 0.0178 & 0.9055 & 0.6929 & \textbf{1.5846} & 1.9664 & 2.9438 & 0.0095 & 1.7876 & 2.9438 & 3.9162 \\
 &  & cal & \textbf{0.0132} & \textbf{0.0001} & \textbf{0.9994} & \textbf{0.3234} & 1.6018 & \textbf{0.0079} & \textbf{0.0108} & \textbf{0.0000} & \textbf{0.0074} & \textbf{0.0108} & \textbf{0.0133} \\
\midrule
\multirow[t]{2}{*}{AGNews} & \multirow[t]{2}{*}{GPT-2} & raw & 0.7796 & \cole{0.1844} & 0.6431 & 1.0539 & 2.1339 & \colc{0.8138} & 0.8894 & \cole{0.4352} & \colc{1.0045} & 0.9803 & 0.7857 \\
 &  & cal & \textbf{0.3786} & \cole{\textbf{0.0362}} & \textbf{0.7001} & \textbf{0.9325} & \textbf{1.8136} & \textbf{0.5353} & \textbf{0.5414} & \cole{\textbf{0.1724}} & \textbf{0.7111} & \textbf{0.5802} & \textbf{0.3816} \\
\midrule
\multirow[t]{6}{*}{CIFAR-100} 
 & \multirow[t]{2}{*}{ResNet20} & raw & 0.3148 & 0.1033 & 0.8360 & 0.8203 & 1.5483 & 0.2661 & 0.4497 & 0.1115 & 0.6054 & 0.4811 & 0.3173 \\
 &  & cal & 0.3167 & \textbf{0.0163} & 0.8423 & 0.7128 & 1.3798 & 0.2434 & 0.4299 & 0.1104 & 0.5544 & 0.4650 & 0.3191 \\

 & \multirow[t]{2}{*}{VGG19} & raw & 0.2639 & \cold{0.1968} & 0.8671 & 1.6021 & 2.0075 & 0.3919 & 0.4494 & \cold{0.0824} & 0.9697 & 0.4590 & 0.2660 \\
 &  & cal & 0.2657 & \cold{0.0439} & 0.8570 & \textbf{0.7062} & \textbf{1.3265} & 0.2475 & 0.3746 & \cold{0.0916} & 0.5008 & 0.3941 & 0.2677 \\
 & \multirow[t]{2}{*}{RepVGG} & raw & \textbf{0.2273} & 0.0558 & \textbf{0.8741} & 0.7509 & 1.3781 & 0.1999 & 0.3269 & \textbf{0.0611} & 0.4704 & 0.3476 & \textbf{0.2290} \\
 &  & cal & 0.2277 & 0.0478 & 0.8722 & 0.7297 & 1.3690 & \textbf{0.1982} & \textbf{0.3262} & 0.0618 & \textbf{0.4598} & \textbf{0.3473} & 0.2295 \\
\midrule
\multirow[t]{6}{*}{CIFAR-10} & \multirow[t]{2}{*}{ResNet20} & raw & 0.0822 & 0.0382 & 0.9216 & 0.7942 & 1.5977 & \cola{0.1223} & \colb{0.1319} & 0.0092 & \cola{0.2368} & \colb{0.1407} & 0.0829 \\
 &  & cal & 0.0839 & \textbf{0.0070} & \textbf{0.9233} & \textbf{0.6123} & \textbf{1.4360} & 0.1011 & 0.1243 & 0.0093 & 0.1900 & 0.1364 & 0.0845 \\

 & \multirow[t]{2}{*}{VGG19} & raw & 0.0677 & 0.0504 & 0.9209 & 1.2340 & 1.8889 & \cola{0.1528} & \colb{0.1237} & 0.0075 & \cola{0.3118} & \colb{0.1268} & 0.0682 \\
 &  & cal & 0.0683 & 0.0128 & 0.9184 & 0.6892 & 1.5651 & 0.1033 & 0.1103 & 0.0081 & 0.1804 & 0.1165 & 0.0689 \\

 & \multirow[t]{2}{*}{RepVGG} & raw & \textbf{0.0526} & 0.0318 & 0.9217 & 0.8760 & 1.6854 & 0.0921 & 0.0889 & \textbf{0.0061} & 0.1859 & 0.0936 & \textbf{0.0530} \\
 &  & cal & 0.0528 & 0.0084 & 0.9146 & 0.6380 & 1.4953 & \textbf{0.0736} & \textbf{0.0824} & 0.0069 & \textbf{0.1390} & \textbf{0.0887} & 0.0532 \\
\bottomrule
\end{tabular}
}
\label{tab:classification_subset}
\end{table}

\vspace{-1mm}
\subsection{Analysis on UA generative systems}\label{sec:gen-exps}
\vspace{-1mm}

In this section, we compare UA generative systems for the TriviaQA task \cite{triviaqa}, which consists of trivia questions from various categories. This dataset is commonly used in papers that focus on the development of approaches to produce confidence scores \cite{kadavath2022,kuhn2023semantic,tian-etal-2023-just,lin_2024,duan2024,xie-etal-2024-calibrating}. 
In those works, correctness is evaluated using semantic equivalence to the gold standard answers provided with the dataset. In preliminary experiments for this work, we found that such automatic correctness labels were often wrong. Hence, for this work, we use a small curated subset of TriviaQA with 455 samples for which the correctness of our systems was assessed manually, as described in detail in Appendix \ref{app:triviaqa_subsample}, where we also show that these errors greatly affect the ranking of systems based on both standard metrics and our proposed metric. To complement the results from this section, Appendix \ref{app:generative} includes extended results for TriviaQA and results  on MMLU \cite{mmlu}, a language understanding dataset.  

We compare results for systems that use three ways to elicit a confidence score \cite{tian-etal-2023-just}. \textbf{Sequence Posterior} extracts the joint probability of the generated token sequence. \textbf{Is True} builds a prompt that requests the model to respond whether the proposed answer is correct and extracts the probability of the True/False token. \textbf{Verbalized} explicitly prompts the model to output a confidence probability as a numerical value. The prompts used for each strategy are detailed in Appendix~\ref{sec:prompts}. 

\begin{table}[t]
\centering
\caption{Results for different metrics on Trivia QA using different LLMs to produce the candidate predictions and three methods to obtain the confidence.}
\resizebox{\columnwidth}{!}{%
\begin{tabular}{ll|c|cccc|c|ccc}
\toprule
& & $\tilde d$ & \multicolumn{4}{c|}{$q_e$} & $\tilde d$, $q_e$ & \multicolumn{3}{c}{$\tilde d$, $q_e$} \vspace{1mm}
\\
& & & & & & & & \multicolumn{3}{c}{ECUAS$_n$} \\
\textbf{LLM} & \textbf{Method} & \textbf{ER} & \textbf{ECE} & \textbf{AUC} & \textbf{CE$_{q_e}$} & \textbf{BS$_{q_e}$} & \textbf{AURC} & \textbf{n=0} & \textbf{n=1} & \textbf{n=128} \\
\midrule
\multirow{3}{*}{Qwen3.5 9B} & Seq. Post. & 0.251 & 0.164 & \textbf{0.923} & \textbf{0.403} & \textbf{0.132} & 0.070 & \textbf{0.484} & \textbf{0.382} & 0.252 \\
 & Is True & 0.251 & \textbf{0.109} & 0.889 & 0.422 & 0.139 & \textbf{0.061} & 0.496 & 0.389 & 0.253 \\
 & Verbalized & \textbf{0.246} & 0.203 & 0.756 & 0.800 & 0.212 & 0.101 & 0.823 & 0.459 & \textbf{0.248} \\
\midrule
\multirow{3}{*}{Ministral-3-8B-Instruct-2512} & Seq. Post. & 0.237 & 0.306 & 0.880 & 0.874 & 0.234 & 0.227 & 0.584 & 0.471 & 0.260 \\
 & Is True & 0.237 & \textbf{0.063} & \textbf{0.895} & \textbf{0.373} & \textbf{0.109} & \textbf{0.073} & \textbf{0.430} & \textbf{0.347} & 0.243 \\
 & Verbalized & \textbf{0.231} & 0.159 & 0.658 & 0.579 & 0.191 & 0.673 & 0.594 & 0.422 & \textbf{0.233} \\
\midrule
\multirow{3}{*}{Gemini 2.5 Flash Lite} & Seq. Post. & 0.119 & \textbf{0.043} & \textbf{0.891} & \textbf{0.232} & \textbf{0.063} & \colc{0.025} & \cola{\textbf{0.272}} & \textbf{0.182} & 0.120 \\
 & Is True & 0.119 & 0.089 & 0.881 & 0.625 & 0.087 & 0.026 & 0.573 & 0.205 & 0.132 \\
 & Verbalized & \textbf{0.108} & 0.068 & 0.853 & 0.445 & 0.080 & \colc{\textbf{0.024}} & \cola{0.462} & 0.188 & \textbf{0.109} \\
\midrule
\multirow{3}{*}{Gemini 2.5 Flash} & Seq. Post. & 0.064 & 0.066 & \textbf{0.855} & \textbf{0.209} & 0.057 & \textbf{0.015} & \textbf{0.213} & 0.121 & 0.064 \\
 & Is True & 0.064 & 0.053 & 0.813 & 0.588 & \textbf{0.053} & 0.031 & 0.449 & 0.117 & 0.081 \\
 & Verbalized & \textbf{0.059} & \textbf{0.049} & 0.743 & 0.532 & 0.054 & 0.016 & 0.503 & \textbf{0.113} & \textbf{0.062} \\
\bottomrule
\end{tabular}
}

\label{tab:triviaqa_subset}
\end{table}

Table~\ref{tab:triviaqa_subset} shows the results for the same metrics as above, except the ones based on $\qvec$ since generative systems do not explicitly calculate the full predictive distribution. In this case, we do not use normalization since there is no prior class distribution to use as reference for the naive system. 
As for the classification results, we can see that the ranking of methods highly depends on the selected metric. While the Verbalized method consistently has the best ER, it is rarely the best in terms of the quality of the confidences. This poses a conflict when using the evaluation approach based on separate metrics for candidate predictions and confidences: which system should we choose? The one with better predictions or the one with better confidences and, in that case, better in terms of which metric? The AURC  provides one way to integrate both aspects of the system. Yet, as discussed above, this metrics is insensitive to monotonic transformations of the confidences. In particular, it is insensitive to miscalibration caused by overfitting. Hence, for example, AURC shows similar values for the Verbalized approach compared to the Sequence Posterior approach for Gemini 2.5 Flash Lite (cells in \colc{blue}). In contrast, for that LLM, the Verbalized approach is markedly worse than the Sequence Posterior method in terms of \ECUAS{0} (cells in \cola{red}), since those confidences are much better calibrated as well as more discriminative. If the goal is to develop a system that produces confidences that can be interpreted as probabilities of correctness, then the AURC metric cannot be used for evaluation. The \ECUAS{n} metrics with a relatively small $n$ value, on the other hand, adequately address this need, while also integrating the impact of incorrect predictions.

\section{Conclusions}
We introduced \ECUAS{n}, a novel family of metrics grounded in statistical decision theory for evaluating uncertainty-augmented (UA) systems, including classifier and generative model-based systems. UA systems allow users to make informed decisions to accept or reject predictions based on the uncertainty and a given rejection trade-off that depends on the application. The \ECUAS{n} metrics are designed to provide a comprehensive assessment of the quality of the system, integrating the impact of incorrect predictions as well as imperfect uncertainties. 

The general metric expression is constructed as a weighted integral of the cost of Bayes decisions for different rejection costs. As a direct consequence of this construction process,  the resulting metrics are proper scoring rules, which implies that they reward uncertainties that are probabilistically interpretable and optimal for decision making. A specific metric from the family can be selected for the application of interest through the parameter $n$, which controls the trade-off between the cost of incorrect predictions and imperfect uncertainties. For high-stakes scenarios, we recommend using $n=0$, which severely penalizes incorrect decisions made with high certainty. 

Through theoretical analysis and empirical results we illustrate the impact of this parameter, showing that smaller values of $n$ may be most appropriate for high-stakes applications, where making incorrect predictions with high-confidence can have severe consequences. An additional contribution of this work is a theoretical motivation for the reported empirical fact that, in generation tasks, uncertainty should be computed over semantic-equivalence sets -- an insight that directly derives from the metric construction process as the cost of Bayes decisions.

\subsection*{Limitations}
\vspace{-2mm}
The proposed family of metrics can be applied to any UA system. The critical assumption made to construct the metric is that the user will make decisions under the framework of statistical decision theory, aiming to minimize the cost. If, on the other hand, the user plans to determine the optimal rejection threshold empirically, and has no need for interpretable uncertainty scores, then these metrics may not be appropriate. 




\bibliography{bib/confidence_extraction,bib/selective_clsf_metrics,bib/custom,bib/std, bib/all-short-3}

@string{ai = {Artificial Intelligence}}

@string{cl = {Computational Linguistics}}

@string{it = {IEEE Trans.\ Inform.\ Theory}}

@inproceedings{Charoenphakdee2020,
  title={Classification with Rejection Based on Cost-sensitive Classification},
  author={Nontawat Charoenphakdee and Zhenghang Cui and Yivan Zhang and Masashi Sugiyama},
  booktitle={International Conference on Machine Learning},
  year={2020},
  url={https://api.semanticscholar.org/CorpusID:225041187}
}

@InProceedings{nadeem10,
  title = 	 {Accuracy-Rejection Curves (ARCs) for Comparing Classification Methods with a Reject Option},
  author = 	 {Nadeem, Malik Sajjad Ahmed and Zucker, Jean-Daniel and Hanczar, Blaise},
  booktitle = 	 {Proceedings of the third International Workshop on Machine Learning in Systems Biology},
  pages = 	 {65--81},
  year = 	 {2009},
  editor = 	 {Džeroski, Sašo and Guerts, Pierre and Rousu, Juho},
  volume = 	 {8},
  series = 	 {Proceedings of Machine Learning Research},
  address = 	 {Ljubljana, Slovenia},
  month = 	 {05--06 Sep},
  publisher =    {PMLR},
  url = 	 {https://proceedings.mlr.press/v8/nadeem10a.html},
}

@article{Bartlett2008,
  title={Classification with a Reject Option using a Hinge Loss},
  author={Peter L. Bartlett and Marten H. Wegkamp},
  journal={J. Mach. Learn. Res.},
  year={2008},
  volume={9},
  pages={1823-1840},
  url={https://api.semanticscholar.org/CorpusID:16963069}
}

@book{russel_2010,
	author = {Russell, Stuart and Norvig, Peter},
	publisher = {Prentice Hall},
	title = {Artificial Intelligence: A Modern Approach},
	year = 2010}

@book{Hunik_2014,
	author = {Hunink, M. G. Myriam and Weinstein, Milton C. and Wittenberg, Eve and Drummond, Michael F. and Pliskin, Joseph S. and Wong, John B. and Glasziou, Paul P.},
	edition = {2},
	place = {Cambridge},
	publisher = {Cambridge University Press},
	title = {Decision Making in Health and Medicine: Integrating Evidence and Values},
	year = {2014}}

@article{Raiffa_1970,
	author = {Mouchart M. H. Raiffa},
	journal = {Recherches {\'e}conomiques de Louvain},
	number = {5},
	pages = {527--528},
	title = {Decision Analysis. Introductory Lectures on Choices under Uncertainty},
	volume = {36},
	year = {1970}}

@misc{Dyrland_2022,
	author = {Dyrland, Kjetil and Lundervold, Alexander Selvikv{\aa}g and Porta Mana, PierGianLuca},
	howpublished = {arXiv:2302.12006},
	month = may,
	title = {Does the evaluation stand up to evaluation? A first-principle approach to the evaluation of classifiers},
	year = {2022}}

@inproceedings{AGnews,
	address = {New York},
	author = {A. Gulli},
	booktitle = {Special Interest Tracks and Posters of the 14th International Conference on World Wide Web},
	location = {Chiba, Japan},
	pages = {880--881},
	title = {The anatomy of a news search engine},
	year = {2005}}

@inproceedings{zhang2015,
	author = {Zhang, Xiang and Zhao, Junbo and LeCun, Yann},
	booktitle = {Proc.~of NeurIPS},
	title = {Character-level convolutional networks for text classification},
	year = {2015}}

@inproceedings{sst2,
	address = {Seattle, Washington, USA},
	author = {Socher, Richard and Perelygin, Alex and Wu, Jean and Chuang, Jason and Manning, Christopher D. and Ng, Andrew and Potts, Christopher},
	booktitle = {Proceedings of the 2013 Conference on Empirical Methods in Natural Language Processing},
	month = oct,
	pages = {1631--1642},
	title = {Recursive Deep Models for Semantic Compositionality Over a Sentiment Treebank},
	year = {2013}}

@book{savage1972foundations,
	author = {Savage, Leonard J},
	publisher = {Courier Corporation},
	title = {The foundations of statistics},
	year = {1972}}

@book{Kahneman:2011,
	author = {Kahneman, Daniel},
	id = {9910114919702121},
	publisher = {1st ed. New York : Farrar, Straus and Giroux},
	title = {Thinking, fast and slow},
	url = {https://search.library.wisc.edu/catalog/9910114919702121},
	year = {2011}}

@article{Morrison_2012,
	author = {G. Stewart Morrison and P. Rose and C. Zhang},
	journal = {Australian Journal of Forensic Sciences},
	month = {June},
	number = {2},
	pages = {155--167},
	publisher = {Informa {UK} Limited},
	title = {Protocol for the collection of databases of recordings for forensic-voice-comparison research and practice},
	volume = {44},
	year = 2012}

@article{DawidMusio2014,
	author = {Dawid, Alexander Philip and Musio, Monica},
	issn = {2281-695X},
	journal = {METRON},
	month = {Apr},
	number = {2},
	pages = {169--183},
	publisher = {Springer Science and Business Media LLC},
	title = {Theory and applications of proper scoring rules},
	volume = {72},
	year = {2014}}

@inproceedings{guo:17,
	address = {Sydney, Australia},
	author = {Chuan Guo and Geoff Pleiss and Yu Sun and Kilian Q. Weinberger},
	booktitle = {Proc.~of the 34th International Conference on Machine Learning},
	title = {On Calibration of Modern Neural Networks},
	year = {2017}}

@inproceedings{sitwdb,
	author = {M. McLaren and L. Ferrer and D. Castan and A. Lawson},
	crossref = {Interspeech:2016},
	title = {The Speakers in the Wild {(SITW)} Speaker Recognition Database},
	url_paper = {https://www.sri.com/wp-content/uploads/pdf/final2c_the_speakers_in_the_wild_28sitw29_speaker_recognition_database.pdf},
	year = {2016}}

@phdthesis{brummer_thesis,
	author = {N. Br\"ummer},
	school = {Stellenbosch University},
	title = {Measuring, Refining and Calibrating Speaker and Language Information Extracted from Speech},
	year = {2010}}

@misc{morrison2015forensic,
	author = {Morrison, G.S. and Zhang, C. and {Enzinger et. al}, E.},
	howpublished = {\url{http://databases.forensic-voice-comparison. net}},
	title = {Forensic database of voice recordings of 500+ Australian English speakers},
	year = {2015}}

@proceedings{Interspeech:2016,
	address = {San Francisco},
	booktitle = {Proc.~Interspeech},
	month = {Sept.},
	title = {Proc.~Interspeech},
	year = {2016}}

@inproceedings{kuhn2023semantic,
title={Semantic Uncertainty: Linguistic Invariances for Uncertainty Estimation in Natural Language Generation},
author={Lorenz Kuhn and Yarin Gal and Sebastian Farquhar},
booktitle={The Eleventh International Conference on Learning Representations },
year={2023},
url={https://openreview.net/forum?id=VD-AYtP0dve}
}

@misc{kadavath2022,
      title={Language Models (Mostly) Know What They Know}, 
      author={Saurav Kadavath and Tom Conerly and Amanda Askell and Tom Henighan and Dawn Drain and Ethan Perez and Nicholas Schiefer and Zac Hatfield-Dodds and Nova DasSarma and Eli Tran-Johnson and Scott Johnston and Sheer El-Showk and Andy Jones and Nelson Elhage and Tristan Hume and Anna Chen and Yuntao Bai and Sam Bowman and Stanislav Fort and Deep Ganguli and Danny Hernandez and Josh Jacobson and Jackson Kernion and Shauna Kravec and Liane Lovitt and Kamal Ndousse and Catherine Olsson and Sam Ringer and Dario Amodei and Tom Brown and Jack Clark and Nicholas Joseph and Ben Mann and Sam McCandlish and Chris Olah and Jared Kaplan},
      year={2022},
      eprint={2207.05221},
      archivePrefix={arXiv},
      primaryClass={cs.CL},
      url={https://arxiv.org/abs/2207.05221}, 
}

@article{lin2022teaching,
title={Teaching Models to Express Their Uncertainty in Words},
author={Stephanie Lin and Jacob Hilton and Owain Evans},
journal={Transactions on Machine Learning Research},
year={2022},
url={https://openreview.net/forum?id=8s8K2UZGTZ},
note={}
}

@inproceedings{tian-etal-2023-just,
    title = "Just Ask for Calibration: Strategies for Eliciting Calibrated Confidence Scores from Language Models Fine-Tuned with Human Feedback",
    author = "Tian, Katherine  and
      Mitchell, Eric  and
      Zhou, Allan  and
      Sharma, Archit  and
      Rafailov, Rafael  and
      Yao, Huaxiu  and
      Finn, Chelsea  and
      Manning, Christopher",
    editor = "Bouamor, Houda  and
      Pino, Juan  and
      Bali, Kalika",
    booktitle = "Proceedings of the 2023 Conference on Empirical Methods in Natural Language Processing",
    month = dec,
    year = "2023",
    address = "Singapore",
    publisher = "Association for Computational Linguistics",
    url = "https://aclanthology.org/2023.emnlp-main.330/",
    doi = "10.18653/v1/2023.emnlp-main.330",
    pages = "5433--5442"
}

@inproceedings{lin_2024,
    title = "Contextualized Sequence Likelihood: Enhanced Confidence Scores for Natural Language Generation",
    author = "Lin, Zhen  and
      Trivedi, Shubhendu  and
      Sun, Jimeng",
    booktitle = "Proceedings of the 2024 Conference on Empirical Methods in Natural Language Processing",
    month = nov,
    year = "2024",
    address = "Miami, Florida, USA",
    url = "https://aclanthology.org/2024.emnlp-main.578/",
}

@inproceedings{duan2024,
    title = "Shifting Attention to Relevance: Towards the Predictive Uncertainty Quantification of Free-Form Large Language Models",
    author = "Duan, Jinhao  and
      Cheng, Hao  and
      Wang, Shiqi  and
      Zavalny, Alex  and
      Wang, Chenan  and
      Xu, Renjing  and
      Kailkhura, Bhavya  and
      Xu, Kaidi",
    booktitle = "Proceedings of the 62nd Annual Meeting of the Association for Computational Linguistics",
    month = aug,
    year = "2024",
    address = "Bangkok, Thailand",
    url = "https://aclanthology.org/2024.acl-long.276/",
}

@inproceedings{triviaqa,
    title = "{T}rivia{QA}: A Large Scale Distantly Supervised Challenge Dataset for Reading Comprehension",
    author = "Joshi, Mandar  and
      Choi, Eunsol  and
      Weld, Daniel  and
      Zettlemoyer, Luke",
    editor = "Barzilay, Regina  and
      Kan, Min-Yen",
    booktitle = "Proceedings of the 55th Annual Meeting of the Association for Computational Linguistics (Volume 1: Long Papers)",
    month = jul,
    year = "2017",
    address = "Vancouver, Canada",
    publisher = "Association for Computational Linguistics",
    url = "https://aclanthology.org/P17-1147/",
    doi = "10.18653/v1/P17-1147",
    pages = "1601--1611"
}

@inproceedings{
mmlu,
title={Measuring Massive Multitask Language Understanding},
author={Dan Hendrycks and Collin Burns and Steven Basart and Andy Zou and Mantas Mazeika and Dawn Song and Jacob Steinhardt},
booktitle={International Conference on Learning Representations},
year={2021},
url={https://openreview.net/forum?id=d7KBjmI3GmQ}
}

@inproceedings{survey-llm-nacl-2024,
    title = "A Survey of Confidence Estimation and Calibration in Large Language Models",
    author = "Geng, Jiahui  and
      Cai, Fengyu  and
      Wang, Yuxia  and
      Koeppl, Heinz  and
      Nakov, Preslav  and
      Gurevych, Iryna",
    editor = "Duh, Kevin  and
      Gomez, Helena  and
      Bethard, Steven",
    booktitle = "Proceedings of the 2024 Conference of the North American Chapter of the Association for Computational Linguistics: Human Language Technologies (Volume 1: Long Papers)",
    month = jun,
    year = "2024",
    address = "Mexico City, Mexico",
    publisher = "Association for Computational Linguistics",
    url = "https://aclanthology.org/2024.naacl-long.366/",
    doi = "10.18653/v1/2024.naacl-long.366",
    pages = "6577--6595"
}

@inproceedings{survey-llm-acl-2025,
    title = "A Survey of Uncertainty Estimation Methods on Large Language Models",
    author = "Xia, Zhiqiu  and
      Xu, Jinxuan  and
      Zhang, Yuqian  and
      Liu, Hang",
    editor = "Che, Wanxiang  and
      Nabende, Joyce  and
      Shutova, Ekaterina  and
      Pilehvar, Mohammad Taher",
    booktitle = "Findings of the Association for Computational Linguistics: ACL 2025",
    month = jul,
    year = "2025",
    address = "Vienna, Austria",
    publisher = "Association for Computational Linguistics",
    url = "https://aclanthology.org/2025.findings-acl.1101/",
    doi = "10.18653/v1/2025.findings-acl.1101",
    pages = "21381--21396",
    ISBN = "979-8-89176-256-5"
}

@inproceedings{survey-llm-kdd-2025,
author = {Liu, Xiaoou and Chen, Tiejin and Da, Longchao and Chen, Chacha and Lin, Zhen and Wei, Hua},
title = {Uncertainty Quantification and Confidence Calibration in Large Language Models: A Survey},
year = {2025},
isbn = {9798400714542},
publisher = {Association for Computing Machinery},
address = {New York, NY, USA},
url = {https://doi.org/10.1145/3711896.3736569},
doi = {10.1145/3711896.3736569},
booktitle = {Proceedings of the 31st ACM SIGKDD Conference on Knowledge Discovery and Data Mining V.2},
pages = {6107–6117},
numpages = {11},
location = {Toronto ON, Canada},
series = {KDD '25}
}

@inproceedings{yang-etal-2025-maqa,
    title = "{MAQA}: Evaluating Uncertainty Quantification in {LLM}s Regarding Data Uncertainty",
    author = "Yang, Yongjin  and
      Yoo, Haneul  and
      Lee, Hwaran",
    editor = "Chiruzzo, Luis  and
      Ritter, Alan  and
      Wang, Lu",
    booktitle = "Findings of the Association for Computational Linguistics: NAACL 2025",
    month = apr,
    year = "2025",
    address = "Albuquerque, New Mexico",
    publisher = "Association for Computational Linguistics",
    url = "https://aclanthology.org/2025.findings-naacl.325/",
    doi = "10.18653/v1/2025.findings-naacl.325",
    pages = "5861--5878",
    ISBN = "979-8-89176-195-7"
}

@article{
lin2024generating,
title={Generating with Confidence: Uncertainty Quantification for Black-box Large Language Models},
author={Zhen Lin and Shubhendu Trivedi and Jimeng Sun},
journal={Transactions on Machine Learning Research},
issn={2835-8856},
year={2024},
url={https://openreview.net/forum?id=DWkJCSxKU5},
note={}
}

@inproceedings{
xiong2024efficient,
title={Efficient and Effective Uncertainty Quantification for {LLM}s},
author={Miao Xiong and Andrea Santilli and Michael Kirchhof and Adam Golinski and Sinead Williamson},
booktitle={Neurips Safe Generative AI Workshop 2024},
year={2024},
url={https://openreview.net/forum?id=QKRLH57ATT}
}

@inproceedings{fadeeva-etal-2024-fact,
    title = "Fact-Checking the Output of Large Language Models via Token-Level Uncertainty Quantification",
    author = "Fadeeva, Ekaterina  and
      Rubashevskii, Aleksandr  and
      Shelmanov, Artem  and
      Petrakov, Sergey  and
      Li, Haonan  and
      Mubarak, Hamdy  and
      Tsymbalov, Evgenii  and
      Kuzmin, Gleb  and
      Panchenko, Alexander  and
      Baldwin, Timothy  and
      Nakov, Preslav  and
      Panov, Maxim",
    editor = "Ku, Lun-Wei  and
      Martins, Andre  and
      Srikumar, Vivek",
    booktitle = "Findings of the Association for Computational Linguistics: ACL 2024",
    month = aug,
    year = "2024",
    address = "Bangkok, Thailand",
    publisher = "Association for Computational Linguistics",
    url = "https://aclanthology.org/2024.findings-acl.558/",
    doi = "10.18653/v1/2024.findings-acl.558",
    pages = "9367--9385"
}

@article{stengel-eskin-van-durme-2023-calibrated,
    title = "Calibrated Interpretation: Confidence Estimation in Semantic Parsing",
    author = "Stengel-Eskin, Elias  and
      Van Durme, Benjamin",
    journal = "Transactions of the Association for Computational Linguistics",
    volume = "11",
    year = "2023",
    address = "Cambridge, MA",
    publisher = "MIT Press",
    url = "https://aclanthology.org/2023.tacl-1.69/",
    doi = "10.1162/tacl_a_00598",
    pages = "1213--1231"
}

@article{jiang-etal-2021-know,
    title = "How Can We Know When Language Models Know? On the Calibration of Language Models for Question Answering",
    author = "Jiang, Zhengbao  and
      Araki, Jun  and
      Ding, Haibo  and
      Neubig, Graham",
    editor = "Roark, Brian  and
      Nenkova, Ani",
    journal = "Transactions of the Association for Computational Linguistics",
    volume = "9",
    year = "2021",
    address = "Cambridge, MA",
    publisher = "MIT Press",
    url = "https://aclanthology.org/2021.tacl-1.57/",
    doi = "10.1162/tacl_a_00407",
    pages = "962--977"
}

@inproceedings{gao-etal-2024-spuq,
    title = "{SPUQ}: Perturbation-Based Uncertainty Quantification for Large Language Models",
    author = "Gao, Xiang  and
      Zhang, Jiaxin  and
      Mouatadid, Lalla  and
      Das, Kamalika",
    editor = "Graham, Yvette  and
      Purver, Matthew",
    booktitle = "Proceedings of the 18th Conference of the European Chapter of the Association for Computational Linguistics (Volume 1: Long Papers)",
    month = mar,
    year = "2024",
    address = "St. Julian{'}s, Malta",
    publisher = "Association for Computational Linguistics",
    url = "https://aclanthology.org/2024.eacl-long.143/",
    doi = "10.18653/v1/2024.eacl-long.143",
    pages = "2336--2346"
}

@inproceedings{
liu2024litcab,
title={LitCab: Lightweight Language Model Calibration over Short- and Long-form Responses},
author={Xin Liu and Muhammad Khalifa and Lu Wang},
booktitle={The Twelfth International Conference on Learning Representations},
year={2024},
url={https://openreview.net/forum?id=jH67LHVOIO}
}

@inproceedings{ulmer-etal-2024-calibrating,
    title = "Calibrating Large Language Models Using Their Generations Only",
    author = "Ulmer, Dennis  and
      Gubri, Martin  and
      Lee, Hwaran  and
      Yun, Sangdoo  and
      Oh, Seong",
    editor = "Ku, Lun-Wei  and
      Martins, Andre  and
      Srikumar, Vivek",
    booktitle = "Proceedings of the 62nd Annual Meeting of the Association for Computational Linguistics (Volume 1: Long Papers)",
    month = aug,
    year = "2024",
    address = "Bangkok, Thailand",
    publisher = "Association for Computational Linguistics",
    url = "https://aclanthology.org/2024.acl-long.824/",
    doi = "10.18653/v1/2024.acl-long.824",
    pages = "15440--15459"
}

@inproceedings{kapoor-etal-2024-calibration,
    title = "Calibration-Tuning: Teaching Large Language Models to Know What They Don{'}t Know",
    author = "Kapoor, Sanyam  and
      Gruver, Nate  and
      Roberts, Manley  and
      Pal, Arka  and
      Dooley, Samuel  and
      Goldblum, Micah  and
      Wilson, Andrew",
    editor = {V{\'a}zquez, Ra{\'u}l  and
      Celikkanat, Hande  and
      Ulmer, Dennis  and
      Tiedemann, J{\"o}rg  and
      Swayamdipta, Swabha  and
      Aziz, Wilker  and
      Plank, Barbara  and
      Baan, Joris  and
      de Marneffe, Marie-Catherine},
    booktitle = "Proceedings of the 1st Workshop on Uncertainty-Aware NLP (UncertaiNLP 2024)",
    month = mar,
    year = "2024",
    address = "St Julians, Malta",
    publisher = "Association for Computational Linguistics",
    url = "https://aclanthology.org/2024.uncertainlp-1.1/",
    doi = "10.18653/v1/2024.uncertainlp-1.1",
    pages = "1--14"
}

@article{mielke-etal-2022-reducing,
    title = "Reducing Conversational Agents' Overconfidence Through Linguistic Calibration",
    author = "Mielke, Sabrina J.  and
      Szlam, Arthur  and
      Dinan, Emily  and
      Boureau, Y-Lan",
    editor = "Roark, Brian  and
      Nenkova, Ani",
    journal = "Transactions of the Association for Computational Linguistics",
    volume = "10",
    year = "2022",
    address = "Cambridge, MA",
    publisher = "MIT Press",
    url = "https://aclanthology.org/2022.tacl-1.50/",
    doi = "10.1162/tacl_a_00494",
    pages = "857--872"
}

@inproceedings{10.5555/3692070.3692835,
author = {Hou, Bairu and Liu, Yujian and Qian, Kaizhi and Andreas, Jacob and Chang, Shiyu and Zhang, Yang},
title = {Decomposing uncertainty for large language models through input clarification ensembling},
year = {2024},
publisher = {JMLR.org},
booktitle = {Proceedings of the 41st International Conference on Machine Learning},
articleno = {765},
numpages = {20},
location = {Vienna, Austria},
series = {ICML'24}
}

@inproceedings{
wang2023evaluating,
title={Evaluating Open-{QA} Evaluation},
author={Cunxiang Wang and Sirui Cheng and Qipeng Guo and Yuanhao Yue and Bowen Ding and Zhikun Xu and Yidong Wang and Xiangkun Hu and Zheng Zhang and Yue Zhang},
booktitle={Thirty-seventh Conference on Neural Information Processing Systems Datasets and Benchmarks Track},
year={2023},
url={https://openreview.net/forum?id=UErNpveP6R}
}

@inproceedings{xie-etal-2024-calibrating,
    title = "Calibrating Language Models with Adaptive Temperature Scaling",
    author = "Xie, Johnathan  and
      Chen, Annie S  and
      Lee, Yoonho  and
      Mitchell, Eric  and
      Finn, Chelsea",
    editor = "Al-Onaizan, Yaser  and
      Bansal, Mohit  and
      Chen, Yun-Nung",
    booktitle = "Proceedings of the 2024 Conference on Empirical Methods in Natural Language Processing",
    month = nov,
    year = "2024",
    address = "Miami, Florida, USA",
    publisher = "Association for Computational Linguistics",
    url = "https://aclanthology.org/2024.emnlp-main.1007/",
    doi = "10.18653/v1/2024.emnlp-main.1007",
    pages = "18128--18138"
}

@misc{qwen35blog,
    title = {Qwen3.5: Accelerating Productivity with Native Multimodal Agents},
    url = {https://qwen.ai/blog?id=qwen3.5},
    author = {Qwen Team},
    month = {February},
    year = {2026}
}

@misc{vteam2025glm45vglm41vthinkingversatilemultimodal,
      title={GLM-4.5V and GLM-4.1V-Thinking: Towards Versatile Multimodal Reasoning with Scalable Reinforcement Learning}, 
      author={V Team and Wenyi Hong and Wenmeng Yu and Xiaotao Gu and Guo Wang and Guobing Gan and Haomiao Tang and Jiale Cheng and Ji Qi and Junhui Ji and Lihang Pan and Shuaiqi Duan and Weihan Wang and Yan Wang and Yean Cheng and Zehai He and Zhe Su and Zhen Yang and Ziyang Pan and Aohan Zeng and Baoxu Wang and Bin Chen and Boyan Shi and Changyu Pang and Chenhui Zhang and Da Yin and Fan Yang and Guoqing Chen and Jiazheng Xu and Jiale Zhu and Jiali Chen and Jing Chen and Jinhao Chen and Jinghao Lin and Jinjiang Wang and Junjie Chen and Leqi Lei and Letian Gong and Leyi Pan and Mingdao Liu and Mingde Xu and Mingzhi Zhang and Qinkai Zheng and Sheng Yang and Shi Zhong and Shiyu Huang and Shuyuan Zhao and Siyan Xue and Shangqin Tu and Shengbiao Meng and Tianshu Zhang and Tianwei Luo and Tianxiang Hao and Tianyu Tong and Wenkai Li and Wei Jia and Xiao Liu and Xiaohan Zhang and Xin Lyu and Xinyue Fan and Xuancheng Huang and Yanling Wang and Yadong Xue and Yanfeng Wang and Yanzi Wang and Yifan An and Yifan Du and Yiming Shi and Yiheng Huang and Yilin Niu and Yuan Wang and Yuanchang Yue and Yuchen Li and Yutao Zhang and Yuting Wang and Yu Wang and Yuxuan Zhang and Zhao Xue and Zhenyu Hou and Zhengxiao Du and Zihan Wang and Peng Zhang and Debing Liu and Bin Xu and Juanzi Li and Minlie Huang and Yuxiao Dong and Jie Tang},
      year={2025},
      eprint={2507.01006},
      archivePrefix={arXiv},
      primaryClass={cs.CV},
      url={https://arxiv.org/abs/2507.01006}, 
}

@article{Liu2026Ministral3,
  title={Ministral 3},
  author={Alexander Liu and Kartik Khandelwal and Sandeep Subramanian and Victor Jouault and Abhinav Rastogi and Adrien Sad'e and Alan Jeffares and Albert Q. Jiang and Alexandre Cahill and Alexandre Gavaudan and Alexandre Sablayrolles and Am'elie H'eliou and Amos You and Andy Ehrenberg and Andy Desman Lo and Anton Eliseev and Antonia Calvi and Avinash Sooriyarachchi and Baptiste Bout and Baptiste Rozi{\`e}re and Baudouin De Monicault and Cl{\'e}mence Lanfranchi and Corentin Barreau and Cyprien Courtot and Daniele Grattarola and Darius Dabert and Diego de Las Casas and Elliot Chane-Sane and Faruk Ahmed and Gabrielle Berrada and Gaetan Ecrepont and Gauthier Guinet and Georgii Sergeevich Novikov and Guillaume Kunsch and Guillaume Lample and Guillaume Martin and Gunshi Gupta and Jan Ludziejewski and Jason Rute and Joachim Studnia and Jonas Amar and Jos{\'e}phine Delas and Josselin Somerville Roberts and Karmesh Yadav and Khyathi Raghavi Chandu and Kush Jain and Laurence Aitchison and Laurent Fainsin and L{\'e}onard Blier and Lingxiao Zhao and Louis Martin and Lucile Saulnier and Luyu Gao and Maarten Buyl and Margaret Jennings and Marie Pellat and Mark Prins and Mathieu Poir'ee and Mathilde Guillaumin and Matthieu Dinot and Matthieu Futeral and Maxime Darrin and Maximilian Augustin and Mia Chiquier and Michel Schimpf and Nathan Grinsztajn and Neha Gupta and Nikhil Raghuraman and Olivier Bousquet and Olivier Duchenne and Patricia Wang and Patrick von Platen and Paul Jacob and Paul Wambergue and Paula Kurylowicz and Pavankumar Reddy Muddireddy and Philom{\`e}ne Chagniot and Pierre Stock and Pravesh Agrawal and Quentin Torroba and Romain Sauvestre and Roman Soletskyi and Rupert Menneer and Sagar Vaze and Samuel Barry and Sanchit Gandhi and Siddhant Waghjale and Siddharth Gandhi and Soham Ghosh and Srijan Mishra and Sumukh Aithal and Szymon Antoniak and Teven Le Scao and Th{\'e}o Cachet and Theo Simon Sorg and Thibaut Lavril and Thiziri Nait Saada and Thomas Chabal and Thomas Foubert and Thomas Robert and Thomas Wang and Tim Lawson and Tom Bewley and Tom Edwards and Umar Jamil and Umberto M. Tomasini and Valeriia Nemychnikova and Van Phung and Vincent Maladiere and Virgile Richard and Wassim Bouaziz and Wen-Ding Li and William Marshall and Xinghui Li and Xinyu Yang and Yassine El Ouahidi and Yihan Wang and Yunhao Tang and Zaccharie Ramzi},
  journal={ArXiv},
  year={2026},
  volume={abs/2601.08584},
  url={https://api.semanticscholar.org/CorpusID:284704499}
}

@article{comanici2025gemini,
  title={Gemini 2.5: Pushing the frontier with advanced reasoning, multimodality, long context, and next generation agentic capabilities},
  author={Comanici, Gheorghe and Bieber, Eric and Schaekermann, Mike and Pasupat, Ice and Sachdeva, Noveen and Dhillon, Inderjit and Blistein, Marcel and Ram, Ori and Zhang, Dan and Rosen, Evan and others},
  journal={arXiv preprint arXiv:2507.06261},
  year={2025}
}

@inproceedings{NEURIPS2018_285e19f2,
 author = {Heo, Jay and Lee, Hae Beom and Kim, Saehoon and Lee, Juho and Kim, Kwang Joon and Yang, Eunho and Hwang, Sung Ju},
 booktitle = {Advances in Neural Information Processing Systems},
 editor = {S. Bengio and H. Wallach and H. Larochelle and K. Grauman and N. Cesa-Bianchi and R. Garnett},
 pages = {},
 publisher = {Curran Associates, Inc.},
 title = {Uncertainty-Aware Attention for Reliable Interpretation and Prediction},
 url = {https://proceedings.neurips.cc/paper_files/paper/2018/file/285e19f20beded7d215102b49d5c09a0-Paper.pdf},
 volume = {31},
 year = {2018}
}

@article{kumar2019calibration,
  title={Calibration of encoder decoder models for neural machine translation},
  author={Kumar, Aviral and Sarawagi, Sunita},
  journal={arXiv preprint arXiv:1903.00802},
  year={2019}
}

@article{Farquhar2024,
author = {Farquhar, Sebastian and Kossen, Jannik and Kuhn, Lorenz and Gal, Yarin},
year = {2024},
month = {06},
pages = {625-630},
title = {Detecting hallucinations in large language models using semantic entropy},
volume = {630},
journal = {Nature},
doi = {10.1038/s41586-024-07421-0}
}

@phdthesis{brummer_measuring_2010,
	title = {Measuring, reﬁning and calibrating speaker and language information extracted from speech},
	url = {https://scholar.sun.ac.za/items/1b46805b-2b1e-46aa-83ce-75ede92f0159},
	language = {en},
	school = {University of Stellenbosch},
	author = {Brummer, Niko},
	year = {2010},
}

@ARTICLE{macedo_entropy,
  author={Macêdo, David and Ren, Tsang Ing and Zanchettin, Cleber and Oliveira, Adriano L. I. and Ludermir, Teresa},
  journal={IEEE Transactions on Neural Networks and Learning Systems}, 
  title={Entropic Out-of-Distribution Detection: Seamless Detection of Unknown Examples}, 
  year={2022},
  volume={33},
  number={6},
  pages={2350-2364},
  doi={10.1109/TNNLS.2021.3112897}}

@article{surveyuq_he,
author = {He, Wenchong and Jiang, Zhe and Xiao, Tingsong and Xu, Zelin and Li, Yukun},
title = {A Survey on Uncertainty Quantification Methods for Deep Learning},
year = {2026},
issue_date = {May 2026},
publisher = {Association for Computing Machinery},
address = {New York, NY, USA},
volume = {58},
number = {7},
issn = {0360-0300},
url = {https://doi.org/10.1145/3786319},
doi = {10.1145/3786319},
journal = {ACM Comput. Surv.},
month = feb,
articleno = {179},
numpages = {35}
}

@inproceedings{
malinin2021uncertainty,
title={Uncertainty Estimation in Autoregressive Structured Prediction},
author={Andrey Malinin and Mark Gales},
booktitle={International Conference on Learning Representations},
year={2021},
url={https://openreview.net/forum?id=jN5y-zb5Q7m}
}

@inproceedings{NIPS2017_9ef2ed4b,
 author = {Lakshminarayanan, Balaji and Pritzel, Alexander and Blundell, Charles},
 booktitle = {Advances in Neural Information Processing Systems},
 editor = {I. Guyon and U. Von Luxburg and S. Bengio and H. Wallach and R. Fergus and S. Vishwanathan and R. Garnett},
 pages = {},
 publisher = {Curran Associates, Inc.},
 title = {Simple and Scalable Predictive Uncertainty Estimation using Deep Ensembles},
 url = {https://proceedings.neurips.cc/paper_files/paper/2017/file/9ef2ed4b7fd2c810847ffa5fa85bce38-Paper.pdf},
 volume = {30},
 year = {2017}
}

@article{IEMOCAP,
author = {Busso, Carlos and Bulut, Murtaza and Lee, Chi-Chun and Kazemzadeh, Abe and Mower Provost, Emily and Kim, Samuel and Chang, Jeannette and Lee, Sungbok and Narayanan, Shrikanth},
year = {2008},
month = {12},
pages = {335-359},
title = {IEMOCAP: Interactive emotional dyadic motion capture database},
volume = {42},
journal = {Language Resources and Evaluation},
doi = {10.1007/s10579-008-9076-6}
}

@Techreport{Krizhevsky2009LearningML,
 author = {Krizhevsky, Alex and Hinton, Geoffrey},
 address = {Toronto, Ontario},
 institution = {University of Toronto},
 number = {0},
 publisher = {Technical report, University of Toronto},
 title = {Learning multiple layers of features from tiny images},
 year = {2009},
 title_with_no_special_chars = {Learning multiple layers of features from tiny images},
 url = {https://www.cs.toronto.edu/~kriz/learning-features-2009-TR.pdf}
}

@inproceedings{geifman_selective_2017,
	title = {Selective {Classification} for {Deep} {Neural} {Networks}},
	volume = {30},
	url = {https://papers.nips.cc/paper_files/paper/2017/hash/4a8423d5e91fda00bb7e46540e2b0cf1-Abstract.html},
	urldate = {2026-04-24},
	booktitle = {Advances in {Neural} {Information} {Processing} {Systems}},
	publisher = {Curran Associates, Inc.},
	author = {Geifman, Yonatan and El-Yaniv, Ran},
	year = {2017},
}

@inproceedings{
geifman2018biasreduced,
title={Bias-Reduced Uncertainty Estimation for Deep Neural Classifiers},
author={Yonatan Geifman and Guy Uziel and Ran El-Yaniv},
booktitle={International Conference on Learning Representations},
year={2019},
url={https://openreview.net/forum?id=SJfb5jCqKm},
}

@inproceedings{zhou_novel_2025,
	title = {A {Novel} {Characterization} of the {Population} {Area} {Under} the {Risk} {Coverage} {Curve} ({AURC}) and {Rates} of {Finite} {Sample} {Estimators}},
	url = {https://openreview.net/forum?id=LBBUJkqkOM},
	language = {en},
	urldate = {2026-04-10},
	author = {Zhou, Han and Landeghem, Jordy Van and Popordanoska, Teodora and Blaschko, Matthew B.},
	month = jun,
	year = {2025},
}

@inproceedings{traub_overcoming_2024,
	title = {Overcoming {Common} {Flaws} in the {Evaluation} of {Selective} {Classification} {Systems}},
	url = {https://openreview.net/forum?id=2TktDpGqNM},
	language = {en},
	urldate = {2026-04-23},
	author = {Traub, Jeremias and Bungert, Till J. and Lüth, Carsten T. and Baumgartner, Michael and Maier-Hein, Klaus and Maier-hein, Lena and Jaeger, Paul F.},
	month = nov,
	year = {2024},
}

@article{franc_optimal_2023,
	title = {Optimal {Strategies} for {Reject} {Option} {Classifiers}},
	volume = {24},
	issn = {1533-7928},
	url = {http://jmlr.org/papers/v24/21-0048.html},
	number = {11},
	urldate = {2026-01-15},
	journal = {Journal of Machine Learning Research},
	author = {Franc, Vojtech and Prusa, Daniel and Voracek, Vaclav},
	year = {2023},
	pages = {1--49},
}

@inproceedings{hendrycks_baseline_2017,
	title = {A {Baseline} for {Detecting} {Misclassified} and {Out}-of-{Distribution} {Examples} in {Neural} {Networks}},
	url = {https://openreview.net/forum?id=Hkg4TI9xl},
	language = {en},
	urldate = {2026-03-12},
	author = {Hendrycks, Dan and Gimpel, Kevin},
	month = feb,
	year = {2017},
}

@article{chow_optimum_1957,
	title = {An optimum character recognition system using decision functions},
	volume = {EC-6},
	issn = {0367-9950},
	url = {https://ieeexplore.ieee.org/document/5222035},
	doi = {10.1109/TEC.1957.5222035},
	number = {4},
	urldate = {2026-03-12},
	journal = {IRE Transactions on Electronic Computers},
	author = {Chow, C. K.},
	month = dec,
	year = {1957},
	pages = {247--254},
}

@article{el-yaniv_foundations_2010,
	title = {On the {Foundations} of {Noise}-free {Selective} {Classification}},
	volume = {11},
	issn = {1533-7928},
	url = {http://jmlr.org/papers/v11/el-yaniv10a.html},
	number = {53},
	urldate = {2026-03-11},
	journal = {Journal of Machine Learning Research},
	author = {El-Yaniv, Ran and Wiener, Yair},
	year = {2010},
	pages = {1605--1641},
}

@inproceedings{savage1961foundations,
  title={The foundations of statistics reconsidered},
  author={Savage, Leonard J},
  booktitle={Proceedings of the Fourth Berkeley Symposium on Mathematical Statistics and Probability, Volume 1: Contributions to the Theory of Statistics},
  volume={4},
  pages={575--587},
  year={1961},
  organization={University of California Press}
}

@article{gneiting_strictly_2007,
	title = {Strictly {Proper} {Scoring} {Rules}, {Prediction}, and {Estimation}},
	volume = {102},
	issn = {0162-1459, 1537-274X},
	url = {http://www.tandfonline.com/doi/abs/10.1198/016214506000001437},
	doi = {10.1198/016214506000001437},
	language = {en},
	number = {477},
	urldate = {2024-12-30},
	journal = {Journal of the American Statistical Association},
	author = {Gneiting, Tilmann and Raftery, Adrian E},
	month = mar,
	year = {2007},
	pages = {359--378},
}

@inproceedings{
geifman2018bias,
title={Bias-Reduced Uncertainty Estimation for Deep Neural Classifiers},
author={Yonatan Geifman and Guy Uziel and Ran El-Yaniv},
booktitle={International Conference on Learning Representations},
year={2019},
url={https://openreview.net/forum?id=SJfb5jCqKm},
}

@inproceedings{jaeger2022call,
  title={A Call to Reflect on Evaluation Practices for Failure Detection in Image Classification},
  author={J{\"a}ger, Paul F and L{\"u}th, Carsten and Klein, Lukas and Bungert, Till},
  booktitle={ICLR 2023},
  year={2023}
}

@inproceedings{bungert2023understanding,
  title={Understanding silent failures in medical image classification},
  author={Bungert, Till J and Kobelke, Levin and Jaeger, Paul F},
  booktitle={International Conference on Medical Image Computing and Computer-Assisted Intervention},
  pages={400--410},
  year={2023},
  organization={Springer}
}

@article{cheng2023unified,
  title={Unified Classification and Rejection: A One-versus-All Framework},
  author={Cheng, Zhen and Zhang, Xu-Yao and Liu, Cheng-Lin},
  journal={arXiv preprint arXiv:2311.13355},
  year={2023}
}

@inproceedings{zhu2023openmix,
  title={Openmix: Exploring outlier samples for misclassification detection},
  author={Zhu, Fei and Cheng, Zhen and Zhang, Xu-Yao and Liu, Cheng-Lin},
  booktitle={Proceedings of the IEEE/CVF Conference on Computer Vision and Pattern Recognition},
  pages={12074--12083},
  year={2023}
}

@article{varshney2020towards,
  title={Towards improving selective prediction ability of nlp systems},
  author={Varshney, Neeraj and Mishra, Swaroop and Baral, Chitta},
  journal={arXiv preprint arXiv:2008.09371},
  year={2020}
}

@INPROCEEDINGS{naushad2024super,
  author={Naushad, Junayed and Voiculescu, Irina},
  booktitle={2024 IEEE International Symposium on Biomedical Imaging (ISBI)}, 
  title={Super-Trustscore: Reliable Failure Detection for Automated Skin Lesion Diagnosis}, 
  year={2024},
  volume={},
  number={},
  pages={1-4},
  doi={10.1109/ISBI56570.2024.10635815}}

@inproceedings{van2024beyond,
  title={Beyond document page classification: Design, datasets, and challenges},
  author={Van Landeghem, Jordy and Biswas, Sanket and Blaschko, Matthew and Moens, Marie-Francine},
  booktitle={Proceedings of the IEEE/CVF Winter Conference on Applications of Computer Vision},
  pages={2962--2972},
  year={2024}
}

@inproceedings{van2023document,
  title={Document understanding dataset and evaluation (DUDE)},
  author={Van Landeghem, Jordy and Tito, Rub{\`e}n and Borchmann, {\L}ukasz and Pietruszka, Micha{\l} and Joziak, Pawel and Powalski, Rafal and Jurkiewicz, Dawid and Coustaty, Micka{\"e}l and Anckaert, Bertrand and Valveny, Ernest and others},
  booktitle={Proceedings of the IEEE/CVF International Conference on Computer Vision},
  pages={19528--19540},
  year={2023}
}

@inproceedings{zhu2022rethinking,
  title={Rethinking confidence calibration for failure prediction},
  author={Zhu, Fei and Cheng, Zhen and Zhang, Xu-Yao and Liu, Cheng-Lin},
  booktitle={European Conference on Computer Vision},
  pages={518--536},
  year={2022},
  organization={Springer}
}

@inproceedings{xin2021art,
  title={The art of abstention: Selective prediction and error regularization for natural language processing},
  author={Xin, Ji and Tang, Raphael and Yu, Yaoliang and Lin, Jimmy},
  booktitle={Proceedings of the 59th Annual Meeting of the Association for Computational Linguistics and the 11th International Joint Conference on Natural Language Processing (Volume 1: Long Papers)},
  pages={1040--1051},
  year={2021}
}

@inproceedings{yoshikawa2023selective,
  title={Selective-LAMA: Selective prediction for confidence-aware evaluation of language models},
  author={Yoshikawa, Hiyori and Okazaki, Naoaki},
  booktitle={Findings of the Association for Computational Linguistics: EACL 2023},
  pages={2017--2028},
  year={2023}
}

@inproceedings{ding2020revisiting,
  title={Revisiting the evaluation of uncertainty estimation and its application to explore model complexity-uncertainty trade-off},
  author={Ding, Yukun and Liu, Jinglan and Xiong, Jinjun and Shi, Yiyu},
  booktitle={Proceedings of the IEEE/CVF Conference on Computer Vision and Pattern Recognition Workshops},
  pages={4--5},
  year={2020}
}

@article{zhu2023revisiting,
  title={Revisiting Confidence Estimation: Towards Reliable Failure Prediction},
  author={Zhu, Fei and Zhang, Xu-Yao and Cheng, Zhen and Liu, Cheng-Lin},
  journal={IEEE Transactions on Pattern Analysis and Machine Intelligence},
  year={2023},
  publisher={IEEE}
}

@article{galil2021disrupting,
  title={Disrupting deep uncertainty estimation without harming accuracy},
  author={Galil, Ido and El-Yaniv, Ran},
  journal={Advances in Neural Information Processing Systems},
  volume={34},
  pages={21285--21296},
  year={2021}
}

@article{franc2023optimal,
  title={Optimal strategies for reject option classifiers},
  author={Franc, Vojtech and Prusa, Daniel and Voracek, Vaclav},
  journal={Journal of Machine Learning Research},
  volume={24},
  number={11},
  pages={1--49},
  year={2023}
}

@article{cen2023devil,
  title={The devil is in the wrongly-classified samples: Towards unified open-set recognition},
  author={Cen, Jun and Luan, Di and Zhang, Shiwei and Pei, Yixuan and Zhang, Yingya and Zhao, Deli and Shen, Shaojie and Chen, Qifeng},
  journal={arXiv preprint arXiv:2302.04002},
  year={2023}
}

@inproceedings{xia2022augmenting,
  title={Augmenting softmax information for selective classification with out-of-distribution data},
  author={Xia, Guoxuan and Bouganis, Christos-Savvas},
  booktitle={Proceedings of the Asian Conference on Computer Vision},
  pages={1995--2012},
  year={2022}
}

@inproceedings{
cattelan2023fix,
title={How to Fix a Broken Confidence Estimator: Evaluating Post-hoc Methods for Selective Classification with Deep Neural Networks},
author={Lu{\'\i}s Felipe Prates Cattelan and Danilo Silva},
booktitle={The 40th Conference on Uncertainty in Artificial Intelligence},
year={2024},
url={https://openreview.net/forum?id=IJBWLRCvYX}
}

@inproceedings{
tran2022plex,
title={Plex: Towards Reliability using Pretrained Large Model Extensions},
author={Dustin Tran and Jeremiah Zhe Liu and Michael W Dusenberry and Du Phan and Mark Collier and Jie Ren and Kehang Han and Zi Wang and Zelda E Mariet and Huiyi Hu and Neil Band and Tim G. J. Rudner and Zachary Nado and Joost van Amersfoort and Andreas Kirsch and Rodolphe Jenatton and Nithum Thain and E. Kelly Buchanan and Kevin Patrick Murphy and D. Sculley and Yarin Gal and Zoubin Ghahramani and Jasper Snoek and Balaji Lakshminarayanan},
booktitle={First Workshop on Pre-training: Perspectives, Pitfalls, and Paths Forward at ICML 2022},
year={2022},
url={https://openreview.net/forum?id=6x0gB9gOHFg}
}

@article{kim2023unified,
  title={A unified benchmark for the unknown detection capability of deep neural networks},
  author={Kim, Jihyo and Koo, Jiin and Hwang, Sangheum},
  journal={Expert Systems with Applications},
  volume={229},
  pages={120461},
  year={2023},
  publisher={Elsevier}
}

@inproceedings{
ashukha2020pitfalls,
title={Pitfalls of In-Domain Uncertainty Estimation and Ensembling in Deep Learning},
author={Arsenii Ashukha and Alexander Lyzhov and Dmitry Molchanov and Dmitry Vetrov},
booktitle={International Conference on Learning Representations},
year={2020},
url={https://openreview.net/forum?id=BJxI5gHKDr}
}

@inproceedings{
xia2024understanding,
title={Towards Understanding Why Label Smoothing Degrades Selective Classification and How to Fix It},
author={Guoxuan Xia and Olivier Laurent and Gianni Franchi and Christos-Savvas Bouganis},
booktitle={The Thirteenth International Conference on Learning Representations},
year={2025},
url={https://openreview.net/forum?id=6oWFn6fY4A}
}

@article{good1952,
    author = {Good, I. J.},
    title = {Rational Decisions},
    journal = {Journal of the Royal Statistical Society: Series B (Methodological)},
    volume = {14},
    number = {1},
    pages = {107-114},
    year = {1952},
    month = {01},
    issn = {0035-9246},
    doi = {10.1111/j.2517-6161.1952.tb00104.x},
    url = {https://doi.org/10.1111/j.2517-6161.1952.tb00104.x},
    eprint = {https://academic.oup.com/jrsssb/article-pdf/14/1/107/49093824/jrsssb_14_1_107.pdf},
}

@book{Peterson_2017, place={Cambridge}, edition={2}, series={Cambridge Introductions to Philosophy}, title={An Introduction to Decision Theory}, publisher={Cambridge University Press}, author={Peterson, Martin}, year={2017}, collection={Cambridge Introductions to Philosophy}}

@book{russell2016artificial,
  title={Artificial Intelligence: A Modern Approach},
  author={Russell, S. and Norvig, P.},
  isbn={9781292153964},
  series={Always learning},
  url={https://books.google.com.ar/books?id=XS9CjwEACAAJ},
  year={2016},
  publisher={Pearson}
}

@misc{dyrland2023doesevaluationstandevaluation,
      title={Does the evaluation stand up to evaluation? A first-principle approach to the evaluation of classifiers}, 
      author={K. Dyrland and A. S. Lundervold and P. G. L. Porta Mana},
      year={2023},
      eprint={2302.12006},
      archivePrefix={arXiv},
      primaryClass={cs.LG},
      url={https://arxiv.org/abs/2302.12006}, 
}

@article{
ferrer2025no,
title={No Need for Ad-hoc Substitutes: The Expected Cost is a Principled All-purpose Classification Metric},
author={Luciana Ferrer},
journal={Transactions on Machine Learning Research},
issn={2835-8856},
year={2025},
url={https://openreview.net/forum?id=5PPbvCExZs},
note={}
}

@article{
ferrer2025evaluating,
title={Evaluating Posterior Probabilities:  Decision Theory, Proper Scoring Rules, and Calibration},
author={Luciana Ferrer and Daniel Ramos},
journal={Transactions on Machine Learning Research},
issn={2835-8856},
year={2025},
url={https://openreview.net/forum?id=qbrE0LR7fF},
note={}
}

@article{hendrickson1971proper,
  title={Proper scores for probability forecasters},
  author={Hendrickson, Arlo D and Buehler, Robert J},
  journal={The Annals of Mathematical Statistics},
  pages={1916--1921},
  year={1971},
  publisher={JSTOR}
}

\appendix

\section{Mathematical proofs}\label{app:math}

In this section we provide the proofs for various propositions made throughout the paper. For all of them, we will rely on the following definition of proper scoring rule (PSR):
\begin{definition}
    Let $\Yspace\triangleq \{y_1,\ldots,y_K\}$ be the set of possible classes and $\Simp_K$ the $K-1$ dimensional simplex. A proper scoring-rule is a function $C^*:\Yspace \times \Simp_K \rightarrow \R_{\geq 0}$ such that, for every $\qvec,\pvec \in \Simp_K$ the following property holds:
    \begin{align*}
        \E_{y\sim \pvec}[C^*(y,\pvec)] \leq \E_{y\sim \pvec}[C^*(y,\qvec)]
    \end{align*}
\end{definition}

\subsection{Bayes decisions for $C_\gamma$ and construction of $C^*_\gamma$}
\label{app:bayes_decisions_for_c_gamma}

We first show that a PSR can be constructed by evaluating a cost function in the Bayes decisor. This is a known fact \cite{DawidMusio2014,brummer_measuring_2010}, but we include the proposition and proof here for completeness:
\begin{proposition}\label{prop:psr_construction}
    Let $C:\Yspace\times \Dspace \rightarrow \R_{\geq 0}$ be a cost function defined for every class-decision pair $(y,d)\in\Yspace \times \Dspace$ for a discrete class set $\Yspace$ and a decision set $\Dspace$. Let $d_B: \Simp_K \rightarrow \Dspace$ such that $d_B(\qvec) \in \argmin_{d\in\Dspace} \E_{y\sim \qvec}[C(y,d)]$ for every $\qvec \in \Simp_K$ be the Bayes decision function. Then, for every $\qvec,\pvec \in \Simp_K$ it holds that
    \begin{align*}
        \E_{y\sim\pvec}[C(y,d_B(\pvec))] \leq \E_{y\sim\pvec}[C(y,d_B(\qvec))]
    \end{align*}
    and therefore, the function $C^*:\Yspace \times \Simp_K \rightarrow \R_{\geq 0}$ defined as $C^*(y,\qvec)\triangleq C^*(y,d_B(\qvec))$ is a proper scoring-rule.
\end{proposition}
\begin{proof}
    Given a $\pvec\in\Simp_K$:
    \begin{align*}
        \E_{y\sim\pvec}[C(y,d_B(\pvec))] \leq \E_{y\sim\pvec}[C(y,d)], \forall d \in \Dspace
    \end{align*}
since, by definition, $d_B(\pvec)$ minimizes $\E_{y\sim \pvec}[C(y,d)]$ for every $d\in\Dspace$. In particular, the decision $d_B(\qvec)$, for $\qvec \in \Simp_K$, should also be such that $\E_{y\sim\pvec}[C(y,d_B(\pvec))] \leq \E_{y\sim\pvec}[C(y,d_B(\qvec))]$, which is the defining property of a proper scoring-rule.
\end{proof}

Now that we know that a PSR $C^*(y,\qvec)$ can be constructed from a cost $C$ and the corresponding Bayes decision function $d_B$, we can obtain a PSR for the cost in \eqref{eq:reject_cost}.
\begin{proposition}\label{thrm:ps_gamma_fixed}
    Let $\Yspace=\{y_1,\ldots,y_K\}$, $\tilde \Dspace$ be a set of decisions, and $\Dspace = \tilde \Dspace \cup \{d_r\}$ the set of decisions that include the rejection option $d_r$. Let also $\tilde C : \Yspace \times \tilde \Dspace \rightarrow \R_{\geq 0}$ be a cost function and $\tilde d_B: \Simp_K \rightarrow \tilde \Dspace$ such that $\tilde{d}_B(\qvec) \in \argmin_{d \in \tilde{\Dspace}} \sum_{k=1}^K q_k \tilde{C}(y_k,d)$ and $u_{\tilde C}: \Simp_K \rightarrow \R_{\geq 0}$ such that $u_{\tilde C}(\qvec)\triangleq \sum_{k=1}^K q_k \tilde C^*(y_k,\qvec)$ with $\tilde C^*(y_k,\qvec) \triangleq \tilde C(y_k,\tilde d_B(\qvec))$ for every $\qvec\in\Simp_K$. Then, the Bayes decision function $d_B:\Simp_K \rightarrow \Dspace$ for the cost $C_\gamma:\Yspace \times \Dspace \rightarrow \R_{\geq 0}$ defined in \eqref{eq:reject_cost}, repeated here for ease of reading:
\begin{align}
    C_{\gamma}(y,d) = 
    \begin{cases}
        \tilde C(y,d) & \mathrm{if}\ \ d \in \tilde \Dspace \\
        \gamma & \mathrm{if}\ \ d = d_r
    \end{cases}
\end{align}
    
    is given by
\begin{align} \label{eqapp:dB}
    d_B(\qvec) =
    \begin{cases}
        \tilde d_B(\qvec) & \mathrm{if}\ \ u_{\tilde C}(\qvec) \leq \gamma \\
        d_r & \mathrm{if}\ \ u_{\tilde C}(\qvec) > \gamma 
    \end{cases}
\end{align}
and the corresponding PSR associated with this cost is $C^*:\Yspace \times \Simp_K \rightarrow \R_{\geq 0}$ such that
\begin{align}\label{eqapp:Cgammapsr}
C_\gamma^*(y,\qvec)\triangleq C_\gamma(y, d_B(\qvec)) = 
    \begin{cases}
        \tilde C^*(y,\qvec) & \mathrm{if}\ \ u_{\tilde C}(\qvec) \leq \gamma \\
        \gamma & \mathrm{if}\ \ u_{\tilde C}(\qvec) > \gamma 
    \end{cases}
\end{align}
\end{proposition}
\begin{proof}
We start by computing the Bayes decision for the cost in \eqref{eq:reject_cost} for some $\qvec\in\Simp_K$ and $\gamma\geq0$.
\begin{align*}
    d_B(\qvec) &=\argmin_{d\in\Dspace}\E[C_\gamma(y,d)] \\
    &= \argmin_{d\in\Dspace}\sum_{k=1}^K q_kC_\gamma(y_k,d) \\
    &= \argmin_{d\in\Dspace}\left(\sum_{k=1}^K q_k C_\gamma(y_k,d) \ind{d \neq d_r} + \sum_{k=1}^K q_k C_\gamma(y_k,d) \ind{d = d_r} \right)\\
    &= \argmin_{d\in\Dspace}\left(\E[\tilde C(y,d)] \ind{d \neq d_r} + \gamma \ind{d = d_r} \right)\\
    & = \begin{dcases}
        \argmin_{d\in\tilde \Dspace}\E[\tilde C(y,d)] & \text{if } \min_{d\in\tilde \Dspace}\E[\tilde C(y,d)] \leq \gamma \\
        d_r & \text{if } \min_{d\in\tilde \Dspace}\E[\tilde C(y,d)] > \gamma 
    \end{dcases} \\
\end{align*}
where all expectations here are taken with respect to $\qvec$. 

We then note that 
\begin{align*}
u_{\tilde C}(\qvec) & \triangleq \sum_{k=1}^K q_k \tilde C^*(y_k,\qvec) \\
& = \sum_{k=1}^K q_k \tilde C(y_k,\tilde d_B(\qvec)) \\
& = \min_{d\in\tilde \Dspace}\E[\tilde C(y,d)]
\end{align*}
and therefore, since $\tilde{d}_B(\qvec) \in \argmin_{d \in \tilde{\Dspace}} \sum_{k=1}^K q_k \tilde{C}(y_k,d) = \argmin_{d\in\tilde \Dspace}\E[\tilde C(y,d)]$, it holds that
\begin{align*}
    d_B(\qvec)& = \begin{dcases}
        \tilde d_B(\qvec) & \text{if  } u_{\tilde C}(\qvec) \leq \gamma \\
        d_r & \text{if  } u_{\tilde C}(\qvec) > \gamma 
    \end{dcases} 
\end{align*}
which proves \eqref{eqapp:dB}. Finally, using Proposition~\ref{prop:psr_construction}, we obtain the PSR $C_\gamma^*(y,\qvec)$ by evaluating $C_\gamma(y,d)$ on the decision rule $d=d_B(\qvec)$. If $u_{\tilde C}(\qvec) \leq \gamma$, then $d_B(\qvec)=\tilde d_B(\qvec)$ and therefore $C_\gamma^*(y,\qvec) = C_\gamma(y,d_B(\qvec))=C_\gamma(y,\tilde d_B(\qvec))=\tilde C(y,\tilde d_B(\qvec))=\tilde C^*(y,\qvec)$. If $u_{\tilde C}(\qvec) > \gamma$, then $d_B(\qvec)=d_r$ and therefore $C_\gamma^*(y,\qvec) = C_\gamma(y,d_B(\qvec))=C_\gamma(y,d_r)=\gamma$. This proves \eqref{eqapp:Cgammapsr}.
\end{proof}

\subsection{$C_w^*$, a weighted integral of $C_\gamma^*$, is a PSR}\label{sec:proof_psr}

The following proposition has been demonstrated for a more general case in \cite{brummer_measuring_2010}. Here we include a self-contained simple proof for our specific scenario.
\begin{proposition}
    Let $C^*_\gamma$, a family of non-negative PSRs parameterized by  $\gamma \in [\gamma_m, \gamma_M] \subset \R$. Let $w:[\gamma_m, \gamma_M] \rightarrow \R_{\geq 0}$ be a non-negative function such that $w(\gamma) C^*_\gamma(y,\qvec)$ is an integrable function of  $\gamma$ for any $y \in \Yspace$ and $\qvec \in \Simp_K$ and $\E_{y\sim \pvec}[C^*(y,\qvec)]$ for $\qvec\in\Simp_K$ is finite for every $\pvec\in\Simp_K$. Then, the function $C^*_w:\Yspace \times \Simp_K \rightarrow \R_{\geq 0}$ defined as
\begin{align} \label{eqapp:Cw}
    C^*_w(y,\qvec) \triangleq \int_{\gamma_m}^{\gamma_M} w(\gamma)C^*_\gamma(y,\qvec)d\gamma
\end{align}
is also a non-negative PSR. 
\end{proposition}

\begin{proof}
    Since $C^*_\gamma$ is a PSR and the expectation with respect to any distribution in $\Simp_K$ is finite, it must satisfy $\E_{y\sim \pvec}[C^*_\gamma(y,\pvec)]\leq \E_{y\sim \pvec}[C^*_\gamma(y,\qvec)]$ for all $\qvec,\pvec \in \Simp_K$. Then, given that $w(\gamma) \geq 0$ for all $\gamma$ and that $w(\gamma) C^*_\gamma(y,\pvec)$ is an integrable function of $\gamma$, we can multiply the above inequality by $w(\gamma)$ and integrate over $\gamma$ on both sides preserving the inequality:
    \begin{align*}
        \E_{y\sim \pvec}[C^*_\gamma(y,\pvec)] &\leq \E_{y\sim \pvec}[C^*_\gamma(y,\qvec)] \\
        \int_{\gamma_m}^{\gamma_M} w(\gamma)\E_{y\sim \pvec}[C^*_\gamma(y,\pvec)] d\gamma&\leq \int_{\gamma_m}^{\gamma_M} w(\gamma)\E_{y\sim \pvec}[C^*_\gamma(y,\qvec)] d\gamma \\
    \end{align*}
As $C^*_\gamma$ is pointwise non-negative, $w(\gamma)C^*_\gamma(y,\qvec)$ is integrable and the expectation $\E_{y\sim \qvec}[C^*_\gamma(y,\qvec)]$ is finite, we can invoke Tornelli's theorem to take out the expectation on both sides of the integral:
    \begin{align*}
        \E_{y\sim \pvec}\left[\int_{\gamma_m}^{\gamma_M} w(\gamma)C^*_\gamma(y,\pvec)d\gamma\right] &\leq \E_{y\sim \pvec}\left[\int_{\gamma_m}^{\gamma_M} w(\gamma)C^*_\gamma(y,\qvec)d\gamma\right] \\
        \E_{y\sim \pvec}\left[C^*_w(y,\pvec)\right] &\leq \E_{y\sim \pvec}\left[C^*_w(y,\qvec)\right] \; \forall \qvec\in\Simp_K
    \end{align*}
Which verifies that $C^*_w$ is a PSR. In addition, since $w(\gamma)$ and $C^*_{\gamma}(y,\qvec)$ is non-negative, the integral is also non-negative.
\end{proof}

\subsection{$C_n^*$: A special case of $C_w^*$}\label{sec:proof_w_n}

Here, we derive \eqref{eq:psr_integral_n} obtained by setting $w(\gamma) = w_n(\gamma)=\alpha_n\gamma^{n-1}$, $\alpha_n=(n+1)u_M^{-(n+1)}$, $\gamma_m=0$, and $\gamma_M = u_M \triangleq \max_{\qvec \in \Simp_K} u_{\tilde C}(\qvec)$ for $n\geq0$ in \eqref{eqapp:Cw} above:
\begin{align}
C^*_n(y,\qvec)&=\int_0^{u_M}w_n(\gamma)C_{\gamma}^*(y,\qvec) d\gamma \nonumber\\
    &=
    \int_0^{u_M} \alpha_n\gamma^{n-1}C_{\gamma}^*(y,\qvec) d\gamma \nonumber\\ 
    &= 
    \alpha_n \int_0^{u_{\tilde C}(\qvec)}\gamma^{n} d\gamma+ \alpha_n\tilde C^*(y,\qvec)\int_{u_{\tilde C}(\qvec)}^{u_M} \gamma^{n-1}d\gamma  \nonumber\\ 
    &= 
    \begin{dcases}\label{eqapp:Cn}
    \frac{\alpha_n}{n+1} u_{\tilde C}(\qvec)^{n+1} + \frac{\alpha_n}{n} \left(u_M^n - u_{\tilde C}(\qvec)^n \right)\tilde C(y,\tilde d_B(\qvec)) & \text{if }  n>0\\ 
    \alpha_n \ u_{\tilde C}(\qvec) + \alpha_n  \left(\log u_M - \log u_{\tilde C}(\qvec)\right)\tilde C(y,\tilde d_B(\qvec)) & \text{if }  n=0\\ 
    \end{dcases}
\end{align}

Note that the derivation of this cost assumes that the uncertainty will never be larger than $u_M$. Yet, in practice, when evaluating this cost on a black box UA system for which internal decisions may be suboptimal for our cost of interest, this condition may not be satisfied. This can be addressed by setting the cost for any uncertainty above $u_M$ to a fixed value equal to the cost when $u_{\tilde C}(\qvec) = u_M$, which is given by the first term in the expressions above. Our code implements this safeguard along with an output message warning the user that their system produces extremely suboptimal uncertainties for the selected $\tilde C$.

\subsection{$C_{n01}^*$: A special case of $C_n^*$ for the 0-1 cost }\label{sec:proof_w_n01}

We can now derive the expression for the special case in which $\tilde C$ is the 0-1 cost. In this case $\Yspace \equiv \tilde \Dspace$, so that $\tilde \Dspace = \{d_1, \ldots, d_K\}$. This cost assumes that there is a one-to-one mapping between decisions and classes so that if the class is $y_k$, the only correct decision is $d_k$:
\begin{align*} 
    \tilde C(y_k,d_j) &= \ind{k \neq j} \\
    \tilde d_B(\qvec) & = \argmin_{d_j \in \tilde \Dspace} \sum_{k=1}^Kq_k\tilde C(y_k,d_j) =\argmin_{d_j \in \tilde\Dspace}  \sum_{k=1}^Kq_k I(k\neq j) =
     d_e \\
    u_{\tilde C}(\qvec) &=\min_{d_j \in \tilde \Dspace}  \sum_{k=1}^Kq_k\tilde C(y_k,d_j) =\min_{j=1,\ldots,K}  \sum_{k=1}^Kq_k I(k\neq j) = \min_{j=1,\ldots,K} 1 - q_j = 1 - q_e \\
    u_M &= \max_{\qvec \in \Simp_K} u_{\tilde C}(\qvec) = 1 - q_m \\
    \alpha_n & = \frac{n+1}{(1-q_m)^{n+1}}
\end{align*}
where $e \triangleq \arg \max_{j=1,\ldots,K} q_j$ is the index of the Bayes decision and $q_m = 1/K$ is the minimum value for $q_e$, achieved when $q_j=1/K$ for all $j$.

Plugging in these values in \eqref{eqapp:Cn}, we get:
\begin{align} \label{eqapp:Cn01}
    C^*_{n01}(y,\qvec) = 
    \begin{dcases}
    \frac{\alpha_n}{n+1} (1-q_e)^{n+1} + \frac{\alpha_n}{n} \left((1-q_m)^n - (1-q_e)^n\right)I(y\neq e) & \text{if }  n>0 \\
    \alpha_n (1-q_e) +  \alpha_n \left(\log (1-q_m) - \log (1-q_e)\right) I(y\neq e) & \text{if }  n=0
    \end{dcases}
\end{align}

\subsection{$C_{n01g}^*$: A special case of $C_n^*$ for the generalized 0-1 cost}
\label{sec:proof_w_n01g}

As explained in section~\ref{sec:generative}, we consider the additional case of $C_{n01g}^*$, the generalized 0-1 cost in which $\tilde C(y,d)=\ind{y\neq \hat y(d)}$. In this case, there is not a one-to-one mapping between elements of $\Yspace$ and $\Dspace$, as assumed above. Instead, we assume that the function $\hat y: \Dspace \rightarrow \Yspace$ maps every possible decision into a corresponding equivalence class. The cost is then determined based on $\hat y(d)$. 

It is easy to see that, for this $\tilde C$,  $u_{\tilde C}(\qvec)$, $u_M$, and $\alpha_n$ are the same as in Appendix~\ref{sec:proof_w_n01}. What changes is the way the Bayes decision should be determined:
\begin{align*} 
    \tilde d_B(\qvec) = \argmin_{d \in \Dspace} \sum_{k=1}^Kq_k\tilde C(y_k,d) = \argmin_{d \in \Dspace} \sum_{k=1}^Kq_k I(y_k\neq \hat y(d)) \in y_e
\end{align*}
where $y_e$ is the equivalence class which has the largest probability, $q_e$. Hence, the Bayes decision in this case is any decision within the most likely equivalence class. In turn, the confidence $q_e$ is not the probability of the most likely decision, but the probability of the most likely equivalence class. 

In general, for tasks in which generative systems are used, we cannot specify $\Yspace$ explicitly. This is not a problem for computing the cost, $\tilde C(y,d)$, which only requires knowing whether d belongs to the same equivalence class as y. Yet, to compute the \eqref{eqapp:Cn} we need $K$, the number of classes. To work around this problem, for the generative scenario we assume that $K \rightarrow \infty$. In practice, though, the cost converges very quickly as $K$ increases, as can be seen in Figure \ref{fig:cncag_vs_qe}, so that assuming $K \rightarrow \infty$ is equivalent to assuming $K$ is larger than 16.

Taking the expression in \ref{eqapp:Cn01} and replacing $K\rightarrow \infty$, this cost results in
\begin{align*}
    &C^*_{n01g}(y,\qvec) = 
    \begin{cases}
    (1-q_e)^{n+1} + \frac{n+1}{n} \left(1 - (1-q_e)^n\right)I(y\neq y_e) & \text{if }  n>0 \\
    (1-q_e) - \log (1-q_e) I(y\neq y_e) & \text{if }  n=0
    \end{cases}
\end{align*}
since
\begin{align*}
    \lim_{K\rightarrow \infty} q_m & = 0 \\
    \lim_{K\rightarrow \infty} \alpha_n & = n+1 \\
    \end{align*}

\section{Classification experiments: setup and additional results}
\label{app:classification_expts}

For the classification experiments, we use the scores for a variety of classification tasks available in \url{https://github.com/luferrer/expected_cost/tree/main/notebooks/data}, produced for the experiments in \cite{ferrer2025evaluating,ferrer2025no}. The datasets include SST2 \citep{sst2}, a natural language processing sentiment analysis dataset, AGNEWS \citep{AGNews,zhang2015}, a news classification dataset, SITW \citep{sitwdb} and FVCAUS \citep{morrison2015forensic,Morrison_2012}, two speaker verification datasets, IEMOCAP \citep{IEMOCAP}, an emotion recognition dataset, and CIFAR10 and CIFAR100 \citep{Krizhevsky2009LearningML}, two image processing datasets where the task is to classify the object in an image into one of 10 or 100 classes, respectively. 
See \cite{ferrer2025evaluating} for a description of how the scores for each dataset were obtained and the links to each dataset.

In our experiments, as also done in the original works \cite{ferrer2025evaluating}, we show results both for raw scores coming out of the models, and for scores calibrated with an affine logistic regression model. This model maps the logarithm of the raw scores generated by the systems using the following expression:
\begin{eqnarray}
    \hat s = \text{softmax}(\alpha \log(s) + \beta) \label{eq:affcal}
\end{eqnarray}
where $s$ is the vector of raw scores produced by the system, $\hat s$ is the calibrated vector of predictive probabilities, $\alpha$ is scalar and $\beta$ is a vector. This transform corresponds to temperature scaling \cite{guo:17} when $\beta$ is close to zero. This happens when the system scores are miscalibrated only due to overfitting, which is often the case for large deep neural networks \cite{guo:17}.
We train calibrator with 5-fold cross-validation, using the code available in \url{https://github.com/luferrer/expected_cost/}.

Table \ref{tab:classification_results}
 shows the results for all datasets. A subset of these results is discussed in Section \ref{sec:experiments_classification}.

The metrics AUC, ECE, CE, BS, and AURC (explained in Section \ref{sec:experiments}) are computed using torch libraries.\footnote{\url{http://torch-uncertainty.github.io/generated/torch_uncertainty.metrics.classification.AURC.html}}
 
\begin{table}[t]
\centering
\caption{Standard and proposed metrics on classification systems. The prefix ``N-'' indicates normalization. The first line in the table indicates the variable that is evaluated by the metrics in each block. ER evaluates the candidate predictions ($\tilde d$), AUC, ECE, BS$_{q_e}$, and CE$_{q_e}$ evaluate different aspects of the confidence ($q_e$), BS$_\qvec$ and CE$_\qvec$ evaluate the full predictive distribution ($\qvec$), and AURC and \ECUAS{n} jointly evaluate the candidate answers and the confidences.}
\resizebox{\columnwidth}{!}{%
\begin{tabular}{lll|c|cccc|cc|c|ccc}
\toprule
& & & $\tilde d$  & \multicolumn{4}{c|}{$q_e$}  & \multicolumn{2}{c|}{$\qvec$} & $\tilde d$, $q_e$ & \multicolumn{3}{c}{$\tilde d$, $q_e$} \\
 &  &  &  &  &  &  &  &  &  &  & \multicolumn{3}{c}{N-\ECUAS{n}} \\
 &  &  & \textbf{N-ER} & \textbf{ECE} & \textbf{AUC} & \textbf{N-CE$_{q_e}$} & \textbf{N-BS$_{q_e}$} & \textbf{N-CE$_\qvec$} & \textbf{N-BS$_\qvec$} & \textbf{AURC} & \textbf{n=0} & \textbf{n=1} & \textbf{n=128} \\
\midrule
\multirow[t]{4}{*}{SST-2} & \multirow[t]{2}{*}{GPT-2 (4-shot)} & raw & 0.9956 & 0.3285 & \textbf{0.9444} & 1.0733 & 2.2368 & 1.0733 & 1.1184 & 0.1783 & 1.0528 & 1.1184 & 1.0015 \\
 &  & cal & \textbf{0.2288} & 0.0163 & 0.8305 & \textbf{0.7867} & 1.6646 & \textbf{0.4032} & \textbf{0.3368} & \textbf{0.0312} & \textbf{0.4350} & \textbf{0.3368} & \textbf{0.2302} \\

 & \multirow[t]{2}{*}{GPT-2 (0-shot)} & raw & 0.8284 & 0.2069 & 0.8059 & 0.9375 & 1.8975 & 0.9172 & 0.9204 & 0.1821 & 0.9162 & 0.9204 & 0.8348 \\
 &  & cal & 0.3080 & \textbf{0.0118} & 0.8136 & 0.7992 & \textbf{1.6466} & 0.4949 & 0.4285 & 0.0500 & 0.5272 & 0.4285 & 0.3103 \\
\midrule
\multirow[t]{2}{*}{SITW} & \multirow[t]{2}{*}{PLDA} & raw & 0.3245 & 0.0012 & \textbf{0.9927} & 0.4951 & 1.6469 & 0.1895 & 0.2681 & 0.0000 & 0.1786 & 0.2681 & 0.3243 \\
 &  & cal & \textbf{0.3045} & \textbf{0.0002} & 0.9925 & \textbf{0.4361} & \textbf{1.6177} & \textbf{0.1580} & \textbf{0.2472} & \textbf{0.0000} & \textbf{0.1456} & \textbf{0.2472} & \textbf{0.3042} \\
\midrule
\multirow[t]{2}{*}{FVCAUS} & \multirow[t]{2}{*}{PLDA} & raw & 3.9155 & 0.0178 & 0.9055 & 0.6929 & \textbf{1.5846} & 1.9664 & 2.9438 & 0.0095 & 1.7876 & 2.9438 & 3.9162 \\
 &  & cal & \textbf{0.0132} & \textbf{0.0001} & \textbf{0.9994} & \textbf{0.3234} & 1.6018 & \textbf{0.0079} & \textbf{0.0108} & \textbf{0.0000} & \textbf{0.0074} & \textbf{0.0108} & \textbf{0.0133} \\
\midrule
\multirow[t]{6}{*}{CIFAR-100} & \multirow[t]{2}{*}{RepVGG-A2} & raw & \textbf{0.2273} & 0.0558 & \textbf{0.8741} & 0.7509 & 1.3781 & 0.1999 & 0.3269 & \textbf{0.0611} & 0.4704 & 0.3476 & \textbf{0.2290} \\
 &  & cal & 0.2277 & 0.0478 & 0.8722 & 0.7297 & 1.3690 & \textbf{0.1982} & \textbf{0.3262} & 0.0618 & \textbf{0.4598} & \textbf{0.3473} & 0.2295 \\

 & \multirow[t]{2}{*}{ResNet-20} & raw & 0.3148 & 0.1033 & 0.8360 & 0.8203 & 1.5483 & 0.2661 & 0.4497 & 0.1115 & 0.6054 & 0.4811 & 0.3173 \\
 &  & cal & 0.3167 & \textbf{0.0163} & 0.8423 & 0.7128 & 1.3798 & 0.2434 & 0.4299 & 0.1104 & 0.5544 & 0.4650 & 0.3191 \\

 & \multirow[t]{2}{*}{VGG19} & raw & 0.2639 & 0.1968 & 0.8671 & 1.6021 & 2.0075 & 0.3919 & 0.4494 & 0.0824 & 0.9697 & 0.4590 & 0.2660 \\
 &  & cal & 0.2657 & 0.0439 & 0.8570 & \textbf{0.7062} & \textbf{1.3265} & 0.2475 & 0.3746 & 0.0916 & 0.5008 & 0.3941 & 0.2677 \\
\midrule
\multirow[t]{2}{*}{Pneum} & \multirow[t]{2}{*}{ResNet-50} & raw & 0.2778 & 0.0763 & 0.8381 & 1.5503 & 1.8887 & 0.8016 & 0.3760 & 0.0314 & 0.9425 & 0.3760 & 0.2777 \\
 &  & cal & \textbf{0.2350} & \textbf{0.0101} & \textbf{0.8597} & \textbf{0.7457} & \textbf{1.6354} & \textbf{0.3361} & \textbf{0.2804} & \textbf{0.0198} & \textbf{0.3605} & \textbf{0.2804} & \textbf{0.2351} \\
\midrule
\multirow[t]{2}{*}{Adrenal} & \multirow[t]{2}{*}{ResNet-50} & raw & 0.9275 & 0.1094 & \textbf{0.8022} & 0.9685 & 1.7767 & 0.9310 & 0.8419 & \textbf{0.0796} & 0.9586 & 0.8419 & 0.9275 \\
 &  & cal & \textbf{0.9130} & \textbf{0.0213} & 0.7798 & \textbf{0.8340} & \textbf{1.6735} & \textbf{0.7950} & \textbf{0.7840} & 0.0868 & \textbf{0.7947} & \textbf{0.7840} & \textbf{0.9152} \\
\midrule
\multirow[t]{2}{*}{Path} & \multirow[t]{2}{*}{ResNet-50} & raw & 0.1137 & 0.0714 & 0.8633 & 1.7322 & 1.8706 & 0.3285 & 0.1859 & 0.0220 & 0.5975 & 0.1918 & 0.1137 \\
 &  & cal & \textbf{0.0758} & \textbf{0.0135} & \textbf{0.9123} & \textbf{0.6585} & \textbf{1.5578} & \textbf{0.1020} & \textbf{0.1118} & \textbf{0.0081} & \textbf{0.1748} & \textbf{0.1187} & \textbf{0.0758} \\
\midrule
\multirow[t]{2}{*}{IEMOCAP} & \multirow[t]{2}{*}{Wav2Vec 2.0 (PT)} & raw & 0.5036 & 0.0629 & \textbf{0.7004} & 0.9427 & 1.8188 & 0.6347 & 0.6464 & 0.2085 & 0.7964 & 0.6810 & 0.5036 \\
 &  & cal & \textbf{0.4970} & \textbf{0.0211} & 0.6978 & \textbf{0.9132} & \textbf{1.7891} & \textbf{0.6152} & \textbf{0.6349} & \textbf{0.2062} & \textbf{0.7690} & \textbf{0.6687} & \textbf{0.4970} \\
\midrule
\multirow[t]{2}{*}{AGNews} & \multirow[t]{2}{*}{GPT-2} & raw & 0.7796 & 0.1844 & 0.6431 & 1.0539 & 2.1339 & 0.8138 & 0.8894 & 0.4352 & 1.0045 & 0.9803 & 0.7857 \\
 &  & cal & \textbf{0.3786} & \textbf{0.0362} & \textbf{0.7001} & \textbf{0.9325} & \textbf{1.8136} & \textbf{0.5353} & \textbf{0.5414} & \textbf{0.1724} & \textbf{0.7111} & \textbf{0.5802} & \textbf{0.3816} \\
\midrule
\multirow[t]{6}{*}{CIFAR-10} & \multirow[t]{2}{*}{ResNet-20} & raw & 0.0822 & 0.0382 & 0.9216 & 0.7942 & 1.5977 & 0.1223 & 0.1319 & 0.0092 & 0.2368 & 0.1407 & 0.0829 \\
 &  & cal & 0.0839 & \textbf{0.0070} & \textbf{0.9233} & \textbf{0.6123} & \textbf{1.4360} & 0.1011 & 0.1243 & 0.0093 & 0.1900 & 0.1364 & 0.0845 \\

 & \multirow[t]{2}{*}{VGG19} & raw & 0.0677 & 0.0504 & 0.9209 & 1.2340 & 1.8889 & 0.1528 & 0.1237 & 0.0075 & 0.3118 & 0.1268 & 0.0682 \\
 &  & cal & 0.0683 & 0.0128 & 0.9184 & 0.6892 & 1.5651 & 0.1033 & 0.1103 & 0.0081 & 0.1804 & 0.1165 & 0.0689 \\

 & \multirow[t]{2}{*}{RepVGG-A2} & raw & \textbf{0.0526} & 0.0318 & 0.9217 & 0.8760 & 1.6854 & 0.0921 & 0.0889 & \textbf{0.0061} & 0.1859 & 0.0936 & \textbf{0.0530} \\
 &  & cal & 0.0528 & 0.0084 & 0.9146 & 0.6380 & 1.4953 & \textbf{0.0736} & \textbf{0.0824} & 0.0069 & \textbf{0.1390} & \textbf{0.0887} & 0.0532 \\
\bottomrule
\end{tabular}
}
\label{tab:classification_results}
\end{table}

\subsection{Impact of the quality of the candidate answers on the ECUAS metrics}

Table \ref{tab:classification_results} shows the impact on the ECUAS metrics of the improvement in the quality of the confidence due to calibration as well as the impact of the quality of decisions, for the datasets for which more than one model is included.
In Figure \ref{fig:cncag_vs_temp} we further explore the impact of the quality of the candidate predictions. To this end, we apply temperature scaling to the calibrated log probabilities, $\qvec_c$ so that the transformed probabilities are given by $\qvec_{t} = \text{softmax}(1/t \log \qvec_c)$, where $t$ is the temperature. We then sample the candidate prediction from the resulting distribution, instead of, as before, selecting the argmax prediction. A temperature close to zero results in always sampling the argmax option, since the probabilities become one-hot, while larger temperatures make the distribution flatter and allow for selection of lower-probability classes. In all cases, the confidence is computed as the original calibrated $\qvec$ for the selected class, simulating a scenario in which the candidate answer is poorly estimated, but the confidence is obtained with an independent model that produces well-calibrated probabilities of correctness.
The figure shows that the metric behaves as expected, incurring a monotonic degradation  for all classification tasks as the temperature increases. 

\begin{figure}
    \centering
    \includegraphics[width=\linewidth]{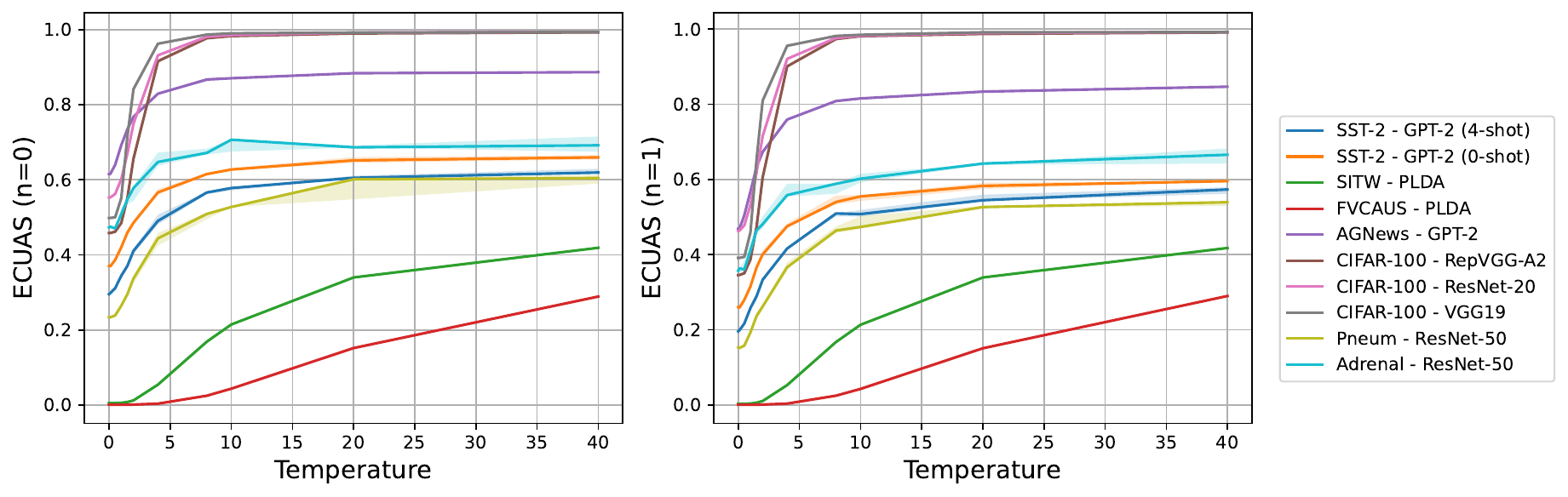}
    \caption{\ECUAS{n} values when temperature scaling is applied to the calibrated version of $\qvec$ and the candidate answer is obtained by sampling from the resulting distribution.}
    \label{fig:cncag_vs_temp}
\end{figure}

\section{Generation results: setup and additional results}\label{app:generative}
In this section, we present details on the setup used for experiments with generative UA systems, as well as additional results on TriviaQA and results on MMLU. 

Our evaluation spans multiple state-of-the-art small LLMs, Qwen 3.5 (4B and 9B) \cite{qwen35blog}, GLM-4.6V-Flash \cite{vteam2025glm45vglm41vthinkingversatilemultimodal}, Ministral-3-8B-Instruct-2512 \cite{Liu2026Ministral3}, as well as larger models from the Gemini 2.5 family (Flash Lite, Flash and Pro) \cite{comanici2025gemini}. We evaluate these models with the methods described in Section \ref{sec:gen-exps} across two datasets: TriviaQA \cite{triviaqa} (Apache 2.0 License) and MMLU (MIT License). Table~\ref{tab:standard_mmlu} and Table~\ref{tab:triviaqa} report the complete performance metrics for all evaluated systems and datasets. A subset of these results is discussed in Section \ref{sec:gen-exps}. We can see that while, in some cases, all metrics agree on which confidence computation method is best for a given LLM, in most cases, the best method depends on the metric. This highlights the importance of choosing the correct metric for the task, one that reflects how useful the system is for decision-making.

\begin{table}[!htbp]
\centering
\caption{Comprehensive performance metrics on the MMLU dataset, comparing candidate predictions generated by various LLMs across three distinct confidence extraction methods.}
\resizebox{\columnwidth}{!}{%
\begin{tabular}{ll|c|cccc|c|ccc}
\toprule
& & $\tilde d$ & \multicolumn{4}{c|}{$q_e$} & $\tilde d$, $q_e$ & \multicolumn{3}{c}{$\tilde d$, $q_e$}  \vspace{1mm}\\
& & & & & & & & \multicolumn{3}{c}{ECUAS$_n$} \\
\textbf{LLM} & \textbf{Method} & \textbf{ER} & \textbf{ECE} & \textbf{AUC} & \textbf{CE$_{q_e}$} & \textbf{BS$_{q_e}$} & \textbf{AURC} & \textbf{n=0} & \textbf{n=1} & \textbf{n=128} \\
\midrule
\multirow{3}{*}{Qwen3.5 4B} & Seq. Post. & \textbf{0.265} & \textbf{0.030} & \textbf{0.824} & \textbf{0.442} & \textbf{0.146} & \textbf{0.093} & \textbf{0.516} & \textbf{0.411} & \textbf{0.267} \\
 & Is True & \textbf{0.265} & 0.092 & 0.494 & 0.623 & 0.207 & 0.272 & 0.654 & 0.472 & \textbf{0.267} \\
 & Verbalized & 0.273 & 0.217 & 0.604 & 2.049 & 0.239 & 0.222 & 2.053 & 0.512 & 0.276 \\
\midrule
\multirow{3}{*}{Qwen3.5 9B} & Seq. Post. & \textbf{0.234} & \textbf{0.046} & \textbf{0.840} & \textbf{0.409} & \textbf{0.131} & \textbf{0.072} & \textbf{0.474} & \textbf{0.365} & \textbf{0.236} \\
 & Is True & \textbf{0.234} & 0.175 & 0.526 & 0.793 & 0.211 & 0.215 & 0.804 & 0.445 & \textbf{0.236} \\
 & Verbalized & \textbf{0.234} & 0.185 & 0.611 & 1.724 & 0.207 & 0.185 & 1.711 & 0.441 & 0.238 \\
\midrule
\multirow{3}{*}{GLM-4.6V-Flash} & Seq. Post. & \textbf{0.264} & \textbf{0.151} & \textbf{0.809} & \textbf{0.591} & \textbf{0.182} & \textbf{0.103} & \textbf{0.640} & \textbf{0.446} & \textbf{0.266} \\
 & Is True & \textbf{0.264} & 0.193 & 0.502 & 0.671 & 0.239 & 0.544 & 0.669 & 0.503 & \textbf{0.266} \\
 & Verbalized & 0.280 & 0.202 & 0.570 & 1.193 & 0.240 & 0.230 & 1.157 & 0.520 & 0.286 \\
\midrule
\multirow{3}{*}{Ministral-3-8B-Instruct-2512} & Seq. Post. & 0.372 & 0.316 & 0.541 & 0.889 & 0.333 & 0.458 & \textbf{0.821} & 0.705 & 0.375 \\
 & Is True & 0.372 & \textbf{0.141} & 0.505 & \textbf{0.771} & \textbf{0.264} & 0.378 & 0.832 & 0.636 & 0.375 \\
 & Verbalized & \textbf{0.327} & 0.234 & \textbf{0.565} & 1.148 & 0.272 & \textbf{0.299} & 1.049 & \textbf{0.599} & \textbf{0.338} \\
\midrule
\multirow{3}{*}{Gemini 2.5 Flash Lite} & Seq. Post. & \textbf{0.181} & \textbf{0.083} & \textbf{0.816} & \textbf{0.422} & \textbf{0.127} & \textbf{0.059} & \textbf{0.455} & \textbf{0.308} & \textbf{0.182} \\
 & Is True & \textbf{0.181} & 0.152 & 0.527 & 1.070 & 0.172 & 0.191 & 1.074 & 0.353 & 0.182 \\
 & Verbalized & 0.187 & 0.104 & 0.624 & 0.599 & 0.156 & 0.123 & 0.616 & 0.343 & 0.188 \\
\midrule
\multirow{3}{*}{Gemini 2.5 Flash} & Seq. Post. & 0.145 & 0.131 & \textbf{0.726} & \textbf{1.003} & 0.135 & 0.085 & \textbf{0.979} & 0.280 & 0.148 \\
 & Is True & 0.145 & 0.135 & 0.547 & 1.346 & 0.143 & 0.138 & 1.347 & 0.288 & 0.146 \\
 & Verbalized & \textbf{0.126} & \textbf{0.094} & 0.646 & 1.054 & \textbf{0.114} & \textbf{0.073} & 1.059 & \textbf{0.240} & \textbf{0.127} \\
\midrule
\multirow{3}{*}{Gemini 2.5 Pro} & Seq. Post. & \textbf{0.087} & 0.179 & 0.622 & 0.545 & 0.128 & 0.108 & \textbf{0.416} & 0.215 & 0.095 \\
 & Is True & \textbf{0.087} & 0.086 & 0.523 & 0.653 & 0.087 & 0.102 & 0.653 & \textbf{0.174} & \textbf{0.088} \\
 & Verbalized & 0.095 & \textbf{0.037} & \textbf{0.669} & \textbf{0.455} & \textbf{0.085} & \textbf{0.058} & 0.461 & 0.180 & 0.096 \\
\bottomrule
\end{tabular}
}
\label{tab:standard_mmlu}
\end{table}
\begin{table}[!htbp]
\centering
\caption{Complete evaluation results on TriviaQA, detailing performance metrics for candidate predictions across multiple LLMs and three different methods for obtaining confidence scores.}
\resizebox{\columnwidth}{!}{%
\begin{tabular}{ll|c|cccc|c|ccc}
\toprule
& & $\tilde d$ & \multicolumn{4}{c|}{$q_e$} & $\tilde d$, $q_e$ & \multicolumn{3}{c}{$\tilde d$, $q_e$}  \vspace{1mm}\\
& & & & & & & & \multicolumn{3}{c}{ECUAS$_n$} \\
\textbf{LLM} & \textbf{Method} & \textbf{ER} & \textbf{ECE} & \textbf{AUC} & \textbf{CE$_{q_e}$} & \textbf{BS$_{q_e}$} & \textbf{AURC} & \textbf{n=0} & \textbf{n=1} & \textbf{n=128} \\
\midrule
\multirow{3}{*}{Qwen3.5 4B} & Seq. Post. & \textbf{0.303} & 0.221 & \textbf{0.922} & \textbf{0.507} & \textbf{0.165} & 0.131 & \textbf{0.585} & \textbf{0.469} & 0.307 \\
 & Is True & \textbf{0.303} & \textbf{0.152} & 0.877 & 0.527 & 0.168 & \textbf{0.086} & 0.612 & 0.471 & \textbf{0.306} \\
 & Verbalized & 0.308 & 0.258 & 0.744 & 2.496 & 0.250 & 0.149 & 2.352 & 0.557 & 0.323 \\
\midrule
\multirow{3}{*}{Qwen3.5 9B} & Seq. Post. & 0.251 & 0.164 & \textbf{0.923} & \textbf{0.403} & \textbf{0.132} & 0.070 & \textbf{0.484} & \textbf{0.382} & 0.252 \\
 & Is True & 0.251 & \textbf{0.109} & 0.889 & 0.422 & 0.139 & \textbf{0.061} & 0.496 & 0.389 & 0.253 \\
 & Verbalized & \textbf{0.246} & 0.203 & 0.756 & 0.800 & 0.212 & 0.101 & 0.823 & 0.459 & \textbf{0.248} \\
\midrule
\multirow{3}{*}{GLM-4.6V-Flash} & Seq. Post. & \textbf{0.209} & \textbf{0.094} & \textbf{0.924} & \textbf{0.338} & \textbf{0.103} & \textbf{0.052} & \textbf{0.389} & \textbf{0.311} & 0.213 \\
 & Is True & \textbf{0.209} & 0.209 & 0.432 & 0.589 & 0.194 & 0.397 & 0.591 & 0.403 & 0.210 \\
 & Verbalized & \textbf{0.209} & 0.131 & 0.790 & 0.576 & 0.154 & 0.070 & 0.611 & 0.363 & \textbf{0.210} \\
\midrule
\multirow{3}{*}{Ministral-3-8B-Instruct-2512} & Seq. Post. & 0.237 & 0.306 & 0.880 & 0.874 & 0.234 & 0.227 & 0.584 & 0.471 & 0.260 \\
 & Is True & 0.237 & \textbf{0.063} & \textbf{0.895} & \textbf{0.373} & \textbf{0.109} & \textbf{0.073} & \textbf{0.430} & \textbf{0.347} & 0.243 \\
 & Verbalized & \textbf{0.231} & 0.159 & 0.658 & 0.579 & 0.191 & 0.673 & 0.594 & 0.422 & \textbf{0.233} \\
\midrule
\multirow{3}{*}{Gemini 2.5 Flash Lite} & Seq. Post. & 0.119 & \textbf{0.043} & \textbf{0.891} & \textbf{0.232} & \textbf{0.063} & 0.025 & \textbf{0.272} & \textbf{0.182} & 0.120 \\
 & Is True & 0.119 & 0.089 & 0.881 & 0.625 & 0.087 & 0.026 & 0.573 & 0.205 & 0.132 \\
 & Verbalized & \textbf{0.108} & 0.068 & 0.853 & 0.445 & 0.080 & \textbf{0.024} & 0.462 & 0.188 & \textbf{0.109} \\
 & Verbalized & \textbf{0.108} & 0.068 & 0.853 & 0.445 & 0.080 & \textbf{0.024} & 0.462 & 0.188 & \textbf{0.109} \\
\midrule
\multirow{3}{*}{Gemini 2.5 Flash} & Seq. Post. & 0.064 & 0.066 & \textbf{0.855} & \textbf{0.209} & 0.057 & \textbf{0.015} & \textbf{0.213} & 0.121 & 0.064 \\
 & Is True & 0.064 & 0.053 & 0.813 & 0.588 & \textbf{0.053} & 0.031 & 0.449 & 0.117 & 0.081 \\
 & Verbalized & \textbf{0.059} & \textbf{0.049} & 0.743 & 0.532 & 0.054 & 0.016 & 0.503 & \textbf{0.113} & \textbf{0.062} \\
\midrule
\multirow{3}{*}{Gemini 2.5 Pro} & Seq. Post. & 0.044 & 0.182 & 0.729 & 0.316 & 0.092 & 0.032 & 0.278 & 0.136 & 0.044 \\
 & Is True & 0.044 & 0.053 & 0.700 & 0.453 & 0.054 & 0.036 & 0.368 & 0.098 & 0.053 \\
 & Verbalized & \textbf{0.037} & \textbf{0.016} & \textbf{0.780} & \textbf{0.191} & \textbf{0.036} & \textbf{0.010} & \textbf{0.185} & \textbf{0.073} & \textbf{0.038} \\
\bottomrule
\end{tabular}
}
\label{tab:triviaqa}
\end{table}

\subsection{TriviaQA Subsample Curation}\label{app:triviaqa_subsample}

The correctness labels for a given prediction in TriviaQA are usually obtained by using a large state-of-the-art LLM to automatically determine whether the predicted answer is semantically equivalent to one of the ground truth answers available in the dataset for that sample \cite{tian-etal-2023-just}. This is done using a Semantic Equivalence Evaluation Prompt (\ref{prompt:semantic-equivalence}). In practice, this process is fragile, often resulting in faulty annotations of correctness.
This observation has been made before in \cite{wang2023evaluating}, where they study the problem by manually annotating answers from five different LLMs for the TriviaQA and the NaturalQuestion datasets. Unfortunately, we were unable to use those annotations for our work since we needed to run the LLMs ourselves to produce various confidence scores. Even if we used the same LLMs as them, we would likely obtain different responses.

To manually annotate the output of our systems, we first selected a random subset of 450 questions which we first manually reviewed to discard problematic cases.  This included questions with ambiguous phrasing due to typos (e.g., "Who first drew Mickey Mouse when ?Disney first supplied the voice?"), time-dependent questions whose answers had changed over time (e.g., ``Who is the father-in-law of Manchester City footballer Sergio Kun Agüerro?", ``Who was the last British male to reach the final of a Grand Slam tennis singles tournament?"), and questions containing nonsensical or inappropriate content in their ground truth answers (including offensive messages, which we have reported on the official HuggingFace webpage for this dataset). After this filter we retained 455 high-quality questions. For these questions, 256 did not need reviewing since they had an exact match to the ground truth for all systems. For the remaining 199 questions, we reviewed the responses from our systems and manually annotated correctness. The curated dataset is publicly available as supplementary material (\url{https://huggingface.co/datasets/ErikErnst/triviaqa-extended-subset}).

The impact of the ground truth curation can be found in Table \ref{tab:llm_judge_comparison}.  The numbers under the E$_\text{LLM}$ (LLM-based semantic equivalence) are obtained with correctness labels produced in two stages. First, exact match to the ground truth answers is attempted. The exact match consists of looking for the normalized string of each of the ground truth answers within the normalized predicted answer. If no exact match is found, then Gemini 3 Flash is used to assess semantic equivalence between the prediction and the ground truth answers. The $E_\text{H}$ numbers are obtained with our manual annotations. These results align with those reported in \cite{wang2023evaluating}. Notably, the ranking of systems greatly differs between the two evaluation procedures for all metrics, in some cases quite markedly -- see, for example, the \ECUAS{0} values for Gemini 2.5 pro. Notably, the impact of the evaluator differs depending on the metric.

\begin{table}[!htbp]
\centering
\caption{Performance according to different metrics when using LLM-based ($E_{\text{LLM}}$), versus\ human ($E_H$) semantic equivalence evaluation on the same set of 455 TriviaQA questions. Best method per LLM in \textbf{bold}. Marked in red are the cases in which evaluators disagree on best method.}
\resizebox{\columnwidth}{!}{%
\begin{tabular}{ll|cc|cc|cc|cc|cc}
\toprule
\textbf{LLM} & \textbf{Method} & \multicolumn{2}{c|}{\textbf{ER}} & \multicolumn{2}{c|}{\textbf{ECE}} & \multicolumn{2}{c|}{AUC} & \multicolumn{2}{c|}{\textbf{AURC}} & \multicolumn{2}{c}{\textbf{ECUAS$_0$}} \\
 & & $E_{\text{LLM}}$ & $E_H$ & $E_{\text{LLM}}$ & $E_H$ & $E_{\text{LLM}}$ & $E_H$ & $E_{\text{LLM}}$ & $E_H$ & $E_{\text{LLM}}$ & $E_H$ \\
\midrule
\multirow{3}{*}{Qwen3.5 4B} & Seq. Post. & 0.336 & \cellcolor{red!15}\textbf{0.303} & 0.188 & 0.221 & \textbf{0.908} & \textbf{0.922} & \cellcolor{red!15}\textbf{0.127} & 0.131 & \textbf{0.610} & \textbf{0.585} \\
 & Is True & 0.336 & \cellcolor{red!15}\textbf{0.303} & \textbf{0.185} & \textbf{0.152} & 0.845 & 0.877 & 0.128 & \cellcolor{red!15}\textbf{0.086} & 0.731 & 0.612 \\
 & Verbalized & \cellcolor{red!15}\textbf{0.334} & 0.308 & 0.284 & 0.258 & 0.722 & 0.744 & 0.170 & 0.149 & 2.742 & 2.352 \\
\midrule
\multirow{3}{*}{Qwen3.5 9B} & Seq. Post. & 0.277 & 0.251 & 0.138 & 0.164 & \textbf{0.911} & \textbf{0.923} & \cellcolor{red!15}\textbf{0.070} & 0.070 & \textbf{0.512} & \textbf{0.484} \\
 & Is True & 0.277 & 0.251 & \textbf{0.136} & \textbf{0.109} & 0.871 & 0.889 & 0.084 & \cellcolor{red!15}\textbf{0.061} & 0.564 & 0.496 \\
 & Verbalized & \textbf{0.273} & \textbf{0.246} & 0.230 & 0.203 & 0.755 & 0.756 & 0.119 & 0.101 & 0.957 & 0.823 \\
\midrule
\multirow{3}{*}{GLM-4.6V-Flash} & Seq. Post. & \textbf{0.246} & \textbf{0.209} & \textbf{0.064} & \textbf{0.094} & \textbf{0.907} & \textbf{0.924} & \textbf{0.056} & \textbf{0.052} & \textbf{0.447} & \textbf{0.389} \\
 & Is True & \textbf{0.246} & \textbf{0.209} & 0.185 & 0.209 & 0.451 & 0.432 & 0.423 & 0.397 & 0.636 & 0.591 \\
 & Verbalized & \textbf{0.246} & \textbf{0.209} & 0.169 & 0.131 & 0.777 & 0.790 & 0.098 & 0.070 & 0.785 & 0.611 \\
\midrule
\multirow{3}{*}{Ministral-3-8B-Instruct-2512} & Seq. Post. & 0.369 & 0.237 & 0.174 & 0.306 & \cellcolor{red!15}\textbf{0.864} & 0.880 & \cellcolor{red!15}\textbf{0.145} & 0.227 & \cellcolor{red!15}\textbf{0.644} & 0.584 \\
 & Is True & 0.369 & 0.237 & 0.188 & \cellcolor{red!15}\textbf{0.063} & 0.807 & \cellcolor{red!15}\textbf{0.895} & 0.160 & \cellcolor{red!15}\textbf{0.073} & 0.800 & \cellcolor{red!15}\textbf{0.430} \\
 & Verbalized & \textbf{0.312} & \textbf{0.231} & \cellcolor{red!15}\textbf{0.135} & 0.159 & 0.651 & 0.658 & 0.611 & 0.673 & 0.744 & 0.594 \\
\midrule
\multirow{3}{*}{Gemini 2.5 Flash Lite} & Seq. Post. & 0.138 & 0.119 & \textbf{0.049} & \textbf{0.043} & \textbf{0.875} & \textbf{0.891} & \cellcolor{red!15}\textbf{0.034} & 0.025 & \textbf{0.325} & \textbf{0.272} \\
 & Is True & 0.138 & 0.119 & 0.109 & 0.089 & 0.820 & 0.881 & 0.050 & 0.026 & 0.791 & 0.573 \\
 & Verbalized & \textbf{0.127} & \textbf{0.108} & 0.088 & 0.068 & 0.791 & 0.853 & 0.037 & \cellcolor{red!15}\textbf{0.024} & 0.719 & 0.462 \\
\midrule
\multirow{3}{*}{Gemini 2.5 Flash} & Seq. Post. & 0.086 & 0.064 & \cellcolor{red!15}\textbf{0.062} & 0.066 & \textbf{0.827} & \textbf{0.855} & 0.026 & \cellcolor{red!15}\textbf{0.015} & \textbf{0.279} & \textbf{0.213} \\
 & Is True & 0.086 & 0.064 & 0.075 & 0.053 & 0.741 & 0.813 & 0.053 & 0.031 & 0.725 & 0.449 \\
 & Verbalized & \textbf{0.075} & \textbf{0.059} & 0.064 & \cellcolor{red!15}\textbf{0.049} & 0.699 & 0.743 & \cellcolor{red!15}\textbf{0.021} & 0.016 & 0.718 & 0.503 \\
\midrule
\multirow{3}{*}{Gemini 2.5 Pro} & Seq. Post. & 0.075 & 0.044 & 0.151 & 0.182 & \cellcolor{red!15}\textbf{0.687} & 0.729 & 0.052 & 0.032 & \cellcolor{red!15}\textbf{0.329} & 0.278 \\
 & Is True & 0.075 & 0.044 & 0.084 & 0.053 & 0.597 & 0.700 & 0.070 & 0.036 & 0.701 & 0.368 \\
 & Verbalized & \textbf{0.073} & \textbf{0.037} & \textbf{0.052} & \textbf{0.016} & 0.686 & \cellcolor{red!15}\textbf{0.780} & \textbf{0.030} & \textbf{0.010} & 0.465 & \cellcolor{red!15}\textbf{0.185} \\
\bottomrule
\end{tabular}
}
\label{tab:llm_judge_comparison}
\end{table}

\subsection{Computational resources}\label{app:resources}
The local experiments detailed in Section \ref{sec:gen-exps} were conducted on a workstation equipped with an AMD Ryzen Threadripper 3960X 24-Core CPU and dual NVIDIA GeForce RTX 3090 GPUs. All inference tasks involving Gemini models were executed via Google Cloud Platform (GCP). To maintain consistent evaluation conditions and manageable computational costs, we disabled the internal reasoning or ``thinking'' modes of the evaluated LLMs whenever supported. Under these configurations, the inference process for each system completed in under 20 minutes. Our experimental codebase is publicly available and designed to seamlessly support both local inference with HuggingFace models and API-based inference through GCP. 

\subsection{Evaluation Prompts}
\label{sec:prompts}

In this section, we provide the exact prompts used for our experiments on TriviaQA and MMLU. For conversational models, these prompts are formatted using the model's specific chat template using the system, user, and assistant roles.

\subsubsection{TriviaQA Prompts}

\textbf{Standard Prompt used to extract the confidences based on Sequence Probability:}
\begin{verbatim}
System: Provide your best guess for the following question. Give ONLY the guess, 
no other words or explanation.
User: Question: {{ question }}
Assistant: Guess:
\end{verbatim}

\textbf{Prompt used to obtained the Verbalized Confidence:}
\begin{verbatim}
System: Provide your best guess and the probability that it is correct (0.0 to 1.0) 
for the following question. Give ONLY the guess and probability, no other words 
or explanation.

Format:
Guess: <short answer>
Probability: <0.0 to 1.0>
User: Question: {{ question }}
Assistant: Guess:
\end{verbatim}

\textbf{Prompt used to extract the confidence as the ``Is True'' probability:}
\begin{verbatim}
System: Evaluate the factual correctness of the proposed answer for the given question.

Response must be only the word 'True' or 'False'. Nothing else, no explanation, 
no other words.
User: Question: {{ question }}
Proposed Answer: {{ answer }}

Is the proposed answer correct?
\end{verbatim}

\textbf{Prompt used to assess Semantic Equivalence for the E$_\text{LLM}$ results:}\label{prompt:semantic-equivalence}
\begin{verbatim}
User: Is the predicted answer to my question semantically equivalent to any 
of the gold standard answers?

Question: {{ question }}
Predicted answer: {{ pred_answer }}
Gold Standard Answers: {{ gold_answers }}

Please answer with a single word, either 'Yes.' or 'No.'
\end{verbatim}

\subsubsection{MMLU Prompts}

\textbf{Standard Prompt used to extract the confidences based on Sequence Probability:}
\begin{verbatim}
System: Answer the following multiple choice question. Respond with ONLY a single 
letter: A, B, C, or D. No other words or explanation.
User: Question: {{ question }}

{{ choices }}
Assistant: Answer:
\end{verbatim}

\textbf{Prompt used to obtained the Verbalized Confidence:}
\begin{verbatim}
System: Answer the following multiple choice question and provide the probability 
that your answer is correct (0.0 to 1.0). Give ONLY the answer letter and probability, 
no other words or explanation.

Format:
Answer: <A, B, C, or D>
Probability: <0.0 to 1.0>
User: Question: {{ question }}

{{ choices }}
Assistant: Answer:
\end{verbatim}

\textbf{Prompt used to extract the confidence as the ``Is True'' probability:}
\begin{verbatim}
System: Evaluate the correctness of the proposed answer for the given multiple 
choice questions.

Response must be only the word 'True' or 'False'. Nothing else, no explanation, 
no other words.
User: Question: {{ question }}

{{ choices }}

Proposed Answer: {{ answer }}

Is the proposed answer correct?
\end{verbatim}

\end{document}